\documentclass{article}

\usepackage[nonatbib, final]{neurips_2019}
\usepackage[numbers]{natbib}

\usepackage{amsmath,amsfonts,bm}

\def\eqref#1{equation~\ref{#1}}

\def\1{\bm{1}}

\DeclareMathAlphabet{\mathsfit}{\encodingdefault}{\sfdefault}{m}{sl}
\SetMathAlphabet{\mathsfit}{bold}{\encodingdefault}{\sfdefault}{bx}{n}

\DeclareMathOperator*{\argmax}{arg\,max}
\DeclareMathOperator*{\argmin}{arg\,min}

\usepackage[utf8]{inputenc} %
\usepackage[T1]{fontenc}    %
\usepackage{hyperref}       %
\usepackage{url}            %
\usepackage{booktabs}       %
\usepackage{amsfonts}       %
\usepackage{nicefrac}       %
\usepackage{microtype}      %
\usepackage{tikz}
\usetikzlibrary{bayesnet}
\usepackage{graphicx}
\usepackage{subcaption}
\usepackage{amsmath, amssymb}
\usepackage{todonotes}
\usepackage{enumitem}
\usepackage{wrapfig}
\usepackage{subcaption}
\usepackage{algorithm}
\usepackage{algorithmic}

\usepackage{amsthm}
\newtheorem{prop}{Proposition}
\newcommand\independent{\protect\mathpalette{\protect\independenT}{\perp}}
\def\independenT#1#2{\mathrel{\rlap{$#1#2$}\mkern2mu{#1#2}}}
\usepackage{amsmath,stackengine}
\stackMath

\newcommand{\method}[1]{{\small\textsc{#1}}}

\def\signed #1{{\leavevmode\unskip\nobreak\hfil\penalty50\hskip2em
  \hbox{}\nobreak\hfil(#1)%
  \parfillskip=0pt \finalhyphendemerits=0 \endgraf}}

\newsavebox\mybox

\title{Causal Confusion in Imitation Learning}

\author{%
   Pim de Haan$^{*}$\textsuperscript{\normalfont 1}, Dinesh Jayaraman$^{\dagger\ddagger}$, Sergey Levine$^\dagger$\\
   $^*$Qualcomm AI Research, University of Amsterdam, \\$^\dagger$Berkeley AI Research, $^\ddagger$ Facebook AI Research
}

\begin{document}
\footnotetext[1]{Work mostly done while at Berkeley AI Research.}

\maketitle

\begin{abstract}
Behavioral cloning reduces policy learning to supervised learning by training a discriminative model to predict expert actions given observations. Such discriminative models are non-causal: the training procedure is unaware of the causal structure of the interaction between the expert and the environment. We point out that ignoring causality is particularly damaging because of the distributional shift in imitation learning. In particular, it leads to a counter-intuitive ``causal misidentification'' phenomenon: access to more information can yield worse performance. We investigate how this problem arises, and propose a solution to combat it through targeted interventions---either environment interaction or expert queries---to determine the correct causal model. We show that causal misidentification occurs in several benchmark control domains as well as realistic driving settings, and validate our solution against DAgger and other baselines and ablations.

\end{abstract}

\section{Introduction}\label{sec:intro}
\vspace{-0.1in}

Imitation learning allows for control policies to be learned directly from example demonstrations provided by human experts. It is easy to implement, and reduces or removes the need for extensive interaction with the environment during training~\cite{widrow1964pattern,pomerleau1989alvinn,bojarski2016end,argall2009survey,hussein2017imitation}.

However, imitation learning suffers from a fundamental problem: distributional shift~\cite{daume2009search,ross2010efficient}. %
Training and testing state distributions are different, induced respectively by the expert and learned policies. Therefore, imitating expert actions on expert trajectories may not align with the true task objective. While this problem is widely acknowledged~\cite{pomerleau1989alvinn,daume2009search,ross2010efficient,ross2011reduction}, yet with careful engineering, na\"{i}ve behavioral cloning approaches have yielded good results for several practical problems~\cite{widrow1964pattern,pomerleau1989alvinn,schaal1999imitation,dave,mulling2013learning,bojarski2016end,mahler2017learning,chauffeurnet}. This raises the question: is distributional shift really still a problem?

In this paper, we identify a somewhat surprising and very problematic effect of distributional shift: ``causal misidentification.'' %
Distinguishing correlates of expert actions in the demonstration set from true causes is usually very difficult, but may be ignored without adverse effects when training and testing distributions are identical (as assumed in supervised learning), since nuisance correlates continue to hold in the test set. %
However, this can cause catastrophic problems in imitation learning due to distributional shift.
This is exacerbated by the causal structure of sequential action: the very fact that current actions cause future observations often introduces complex new nuisance correlates. %

To illustrate, consider behavioral cloning to train a neural network to drive a car. In scenario A, the model's input is an image of the dashboard and windshield, and in scenario B, the input to the model (with identical architecture) is the same image but with the dashboard masked out (see Fig~\ref{fig:causal_misidentification}). Both cloned policies achieve low training loss, but when tested on the road, model B drives well, while model A does not.
The reason: the dashboard has an indicator light that comes on immediately when the brake is applied, and model A wrongly learns to apply the brake only when the brake light is on. Even though the brake light is the \emph{effect} of braking, model A could achieve low training error by \emph{misidentifying} it as the cause instead.

\begin{figure}
    \centering
    \includegraphics[width=1\linewidth]{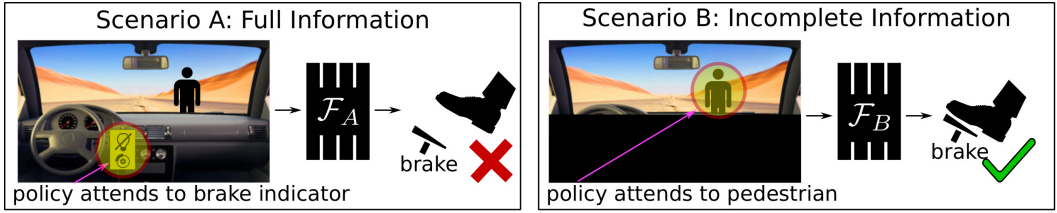}
    \caption{\small Causal misidentification: \emph{more} information yields worse imitation learning performance. %
    Model A
    relies on the braking indicator to decide whether to brake. %
    Model B instead correctly attends to the pedestrian.}
    \label{fig:causal_misidentification}
    \vspace{-0.2in}
\end{figure}

This situation presents a give-away symptom of causal misidentification: access to \emph{more information} leads to \emph{worse generalization performance} in the presence of distributional shift.
Causal misidentification occurs commonly in natural imitation learning settings, especially when the imitator's inputs include history information.

In this paper, we first point out and investigate the causal misidentification problem in imitation learning. 
Then, we propose a solution to overcome it by learning the correct causal model, even when using complex deep neural network policies. 
We learn a mapping from causal graphs to policies, and then use targeted interventions to efficiently search for the correct policy, either by querying an expert, or by executing selected policies in the environment.

\vspace{-0.1in}
\section{Related Work}\label{sec:related}
\vspace{-0.1in}
\noindent \textbf{Imitation learning.}~~~Imitation learning through behavioral cloning dates back to Widrow and Smith, 1964~\cite{widrow1964pattern}, %
and has remained popular through today~\cite{pomerleau1989alvinn,schaal1999imitation,dave,mulling2013learning,bojarski2016end,giusti2016machine,mahler2017learning,Wang2019-us,chauffeurnet}. The distributional shift problem, wherein a cloned policy encounters unfamiliar states during autonomous execution, has been identified as an issue in imitation learning~\cite{pomerleau1989alvinn,daume2009search,ross2010efficient,ross2011reduction,laskey2017dart,ho2016generative,chauffeurnet}. This is closely tied to the ``feedback'' problem in general machine learning systems that have direct or indirect access to their own past states~\cite{sculley2014machine,feedback_video}. For imitation learning, various solutions to this problem have been proposed~\citep{daume2009search,ross2010efficient,ross2011reduction} that rely on iteratively querying an expert based on states encountered by some intermediate cloned policy, to overcome distributional shift; DAgger~\cite{ross2011reduction} has come to be the most widely used of these solutions. 

We show evidence that the distributional shift problem in imitation learning is often due to causal misidentification, as illustrated schematically in Fig~\ref{fig:causal_misidentification}. %
We propose to address this through targeted interventions on the states to learn the true causal model to overcome distributional shift. As we will show, these interventions can take the form of either environmental rewards with no additional expert involvement, or of expert queries in cases where the expert is available for additional inputs. In expert query mode, our approach may be directly compared to DAgger~\citep{ross2011reduction}: indeed, we show that we successfully resolve causal misidentification using orders of magnitude fewer queries than DAgger. 

We also compare against ~\citet{chauffeurnet}: to prevent imitators from copying past actions, they train with dropout~\cite{dropout} on dimensions that might reveal past actions. While our approach seeks to find the true causal graph in a mixture of graph-parameterized policies, dropout corresponds to directly applying the mixture policy. In our experiments, our approach performs significantly better.

\noindent \textbf{Causal inference.}~~~Causal inference is the general problem of deducing cause-effect relationships among variables~\citep{spirtes2000causation,pearl2009causality,peters2017elements,spirtes2010introduction,eberhardt2017introduction,spirtes2016causal}. ``Causal discovery'' approaches allow causal inference from pre-recorded observations under constraints~\citep{steyvers2003inferring,Heckerman2006bayesiancausal,lopez2017discovering,guyon2008design,louizos2017causal,maathuis2010predicting,le2016fast,goudet2017learning,mitrovic2018causal,wang2018blessings}. Observational causal inference is known to be impossible in general~\citep{pearl2009causality,peters2014causal}. We operate in the interventional regime~\citep{tong2001active,eberhardt2007interventions,shanmugam2015learning,sen2017identifying} where a user may ``experiment'' to discover causal structures by assigning values to some subset of the variables of interest and observing the effects on the rest of the system. We propose a new interventional causal inference approach suited to imitation learning. While ignoring causal structure is particularly problematic in imitation learning, ours is the first effort directly addressing this, to our knowledge.

\vspace{-0.1in}
\section{The Phenomenon of Causal Misidentification}\label{sec:problem}
\vspace{-0.1in}

In imitation learning, an expert demonstrates how to perform a task (e.g.,~driving a car) for the benefit of an agent. %
In each demo, the agent has access both to its $n$-dim.~state observations at each time $t$, $X^t=[X^t_1, X^t_2, \ldots X_n^t]$ (e.g.,~a video feed from a camera), and to the expert's action $A^t$ (e.g.,~steering, acceleration, braking). Behavioral cloning approaches learn a mapping $\pi$ from $X^t$ to $A^t$ using all $(X^t, A^t)$ tuples from the demonstrations. At test time, the agent observes $X^t$ and executes $\pi(X^t)$.

\begin{wrapfigure}[10]{r}{0.4\linewidth}
  \vspace{-0.15in}
  \centering
    \includegraphics[width=0.4\textwidth]{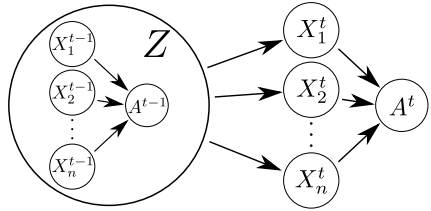}
    \caption{\small Causal dynamics of imitation. Parents of a node represent its causes. %
    }\label{fig:graph}
\end{wrapfigure}
The underlying sequential decision process has complex causal structures, represented in Fig~\ref{fig:graph}. %
States influence future expert actions, and are also themselves influenced by past actions and states.

In particular, expert actions $A^t$ are influenced by \emph{some} information in state $X^t$, and unaffected by the rest.
For the moment, assume that the dimensions $X_1^t, X_2^t, X_3^t, \dots$ of $X^t$ represent disentangled factors of variation.
Then some unknown subset of these factors (``causes'') affect expert actions, and the rest do not (``nuisance variables''). A confounder $Z^t=[X^{t-1}, A^{t-1}]$ influences each state variable in $X^t$, so that some nuisance variables may still be correlated with $A^t$ among $(X^t, A^t)$ pairs from demonstrations. In Fig~\ref{fig:causal_misidentification}, the dashboard light is a nuisance variable. %

A na\"ive behavioral cloned policy might rely on nuisance correlates to select actions, producing low training error, and even generalizing to held-out $(X^t, A^t)$ pairs. However, this policy must contend with distributional shift when deployed: actions $A_t$ are chosen by the \emph{imitator} rather than the expert, affecting the distribution of $Z^t$ and $X^t$.  This in turn affects the policy mapping from $X^t$ to $A^t$, leading to poor performance of expert-cloned policies. We define ``causal misidentification" as the phenomenon whereby cloned policies fail by misidentifying the causes of expert actions.

\vspace{-0.1in}
\subsection{Robustness and Causality in Imitation Learning}\label{sec:causal_dist_shift}

Intuitively, distributional shift affects the relationship of the expert action $A^t$ to nuisance variables, but not to the true causes. In other words, to be maximally robust to distributional shift, a policy must rely solely on the true causes of expert actions, thereby avoiding causal misidentification. This intuition can be formalized in the language of functional causal models (FCM) and interventions~\cite{pearl2009causality}.

\noindent \textbf{Functional causal models:} A functional causal model (FCM) over a set of variables $\{Y_i\}_{i=1}^n$ is a tuple $(G, \theta_{G})$ containing a graph $G$ over $\{Y_i\}_{i=1}^n$, and deterministic functions $f_i(\cdot ; \theta_{G})$ with parameters $\theta_G$ describing how the causes of each variable $Y_i$ determine it:
    $Y_i = f_i(Y_{\text{Pa(i;G)}}, E_i ; \theta_G),$
where $E_i$ is a stochastic noise variable that represents all external influences on $Y_i$, and $\text{Pa}(i; G)$ denote the indices of parent nodes of $Y_i$, which correspond to its causes.

An ``intervention'' $do(Y_i)$ on $Y_i$ to set its value may now be represented by a structural change in this graph to produce the ``mutilated graph'' $G_{\bar{Y_i}}$, in which incoming edges to $Y_i$ are removed.\footnote{For a more thorough overview of FCMs, see~\cite{pearl2009causality}.}

Applying this formalism to our imitation learning setting, any distributional shift in the state $X^t$ may be modeled by intervening on $X^t$, so that correctly modeling the ``interventional query'' $p(A^t|do(X^t))$ is sufficient for robustness to distributional shifts. Now, we may formalize the intuition that only a policy relying solely on true causes can robustly model the mapping from states to optimal/expert actions under distributional shift. 
In Appendix \ref{app:prop}, we prove that under mild assumptions, correctly modeling interventional queries does indeed require learning the correct causal graph $G$. In the car example, ``setting'' the brake light to on or off and observing the expert's actions would yield a clear signal unobstructed by confounders: the brake light does not affect the expert's braking behavior.

\vspace{-0.1in}
\subsection{Causal Misidentification in Policy Learning Benchmarks and Realistic Settings}\label{sec:benchmark}

Before discussing our solution, we first present several testbeds and real-world cases where causal misidentification adversely influences imitation learning performance.

\noindent \textbf{Control Benchmarks.}~~~We show that causal misidentification is induced with small changes to widely studied benchmark control tasks, simply by adding more information to the state, which intuitively ought to make the tasks easier, not harder. In particular, we add information about the previous action, which tends to correlate with the current action in the expert data for many standard control problems. This is a proxy for scenarios like our car example, in which correlates of past actions are observable in the state, and is similar to what we might see from other sources of knowledge about the past, such as memory or recurrence.
We study three kinds of tasks: (i) MountainCar (continuous states, discrete actions), (ii) MuJoCo Hopper (continuous states and actions), (iii) Atari games: Pong, Enduro and UpNDown (states: two stacked consecutive frames, discrete actions).

\begin{wrapfigure}[10]{r}{0.47\linewidth}
  \vspace{-0.15in}
    \begin{subfigure}[b]{0.15\textwidth}
      \includegraphics[width=\textwidth]{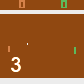}
    \caption{Pong}
    \end{subfigure}
    \begin{subfigure}[b]{0.15\textwidth}
      \includegraphics[width=\textwidth]{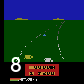}
    \caption{Enduro}
    \end{subfigure}
    \begin{subfigure}[b]{0.15\textwidth}
      \includegraphics[width=\textwidth]{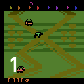}
    \caption{UpNDown}
    \end{subfigure}
        \caption{\small The Atari environments with indicator of past action (white number in lower left).}    \label{fig:environments}
\end{wrapfigure}
For each task, we study imitation learning
in two scenarios. In scenario A (henceforth called "\method{confounded}"),
the policy sees the augmented observation vector, including the previous action. In the case of low-dimensional observations, the state vector is expanded to include the previous action at an index that is unknown to the learner. In the case of image observations, we overlay a symbol corresponding to the previous action at an unknown location on the image (see Fig~\ref{fig:environments}). In scenario B ("\method{original}"), the previous action variable is replaced with random noise for low-dimensional observations. For image observations, the original images are left unchanged. Demonstrations are generated synthetically as described in Appendix~\ref{app:demonstrations}. In all cases, we use neural networks with identical architectures to represent the policies, and we train them on the same demonstrations.
Fig~\ref{fig:diagnosis} shows the rewards against varying demonstration dataset sizes for MountainCar, Hopper, and Pong. Appendix~\ref{app:diagnosis} shows additional results, including for Enduro and UpNDown.  All policies are trained to near-zero validation error on held-out expert state-action tuples. \method{original} produces rewards tending towards expert performance as the size of the imitation dataset increases. \method{confounded} either requires  many more demonstrations to reach equivalent performance, or fails completely to do so.

Overall, the results are clear: across these tasks, access to \emph{more} information leads to inferior performance. As Fig~\ref{fig:diagnosis_full} in the appendix shows, this difference is not due to different training/validation losses on the expert demonstrations---for example, in Pong, \method{confounded} produces lower validation loss than \method{original} on held-out demonstration samples, but produces lower rewards when actually used for control.
These results not only validate the existence of causal misidentification, but also provides us with testbeds for investigating a potential solution. %

\noindent \textbf{Real-World Driving.}~~~
Our testbeds introduce deliberate nuisance variables to the ``original'' observation variables for ease of evaluation, but evidence suggests that misattribution is pervasive in common real-world imitation learning settings. Real-world problems often have no privileged ``original'' observation space, and very natural-seeming state spaces may still include nuisance factors---as in our dashboard light setting (Fig~\ref{fig:causal_misidentification}), where causal misattribution occurs when using the full image from the camera.  

In particular, history would seem a natural part of the state space for real-world driving, yet recurrent/history-based imitation has been consistently observed in prior work to hurt performance, thus exhibiting clear symptoms of causal misidentification~\cite{dave,Wang2019-us,chauffeurnet}.
While these histories contain valuable information for driving, they also naturally introduce information about nuisance factors such as previous actions. In all three cases, more information led to worse results for the behavioral cloning policy, but this was neither attributed specifically to causal misidentification, nor tackled using causally motivated approaches.

\begin{wraptable}{r}{0.5\textwidth}
  \small
  \vspace{-0.1in}
    \centering
    \resizebox{0.5\textwidth}{!}{
    \begin{tabular}{l|c|ccc}
        \toprule
    Metrics $\rightarrow$      & Validation   & \multicolumn{3}{c}{Driving Performance} \\
    Methods $\downarrow$       & Perplexity  & Distance & Interventions & Collisions \\  \midrule
\method{history}      & \textbf{0.834} &  144.92 & 2.94 $\pm$ 1.79 & 6.49 $\pm$ 5.72\\
\method{no-history}  & 0.989 & \textbf{268.95} & \textbf{1.30 $\pm$ 0.78} & \textbf{3.38 $\pm$ 2.55} \\ \bottomrule
\end{tabular}
}
\caption{\small{Imitation learning results from~\citet{Wang2019-us}}. Accessing history yields better validation performance, but worse actual driving performance.}
\label{tab:dequan}
\end{wraptable}
We draw the reader's attention to particularly telling results from~\citet{Wang2019-us} for learning to drive in near-photorealistic GTA-V~\cite{krahenbuhl2018free} environments, using behavior cloning with DAgger-inspired expert perturbation. Imitation learning policies are trained using overhead image observations with and without ``history'' information (\method{history} and \method{no-history}) about the ego-position trajectory of the car in the past. 

Similar to our tests above, architectures are identical for the two methods. And once again, like in our tests above, \method{history} has better performance on held-out demonstration data, but much worse performance when actually deployed.
Tab~\ref{tab:dequan} shows these results, reproduced from~\citet{Wang2019-us} Table II. These results constitute strong evidence for the prevalence of causal misidentification in realistic imitation learning settings. \citet{chauffeurnet} also observe similar symptoms in a driving setting, and present a dropout~\cite{dropout} approach to tackle it, which we compare to in our experiments. Subsequent to an earlier version of this work, \citet{codevilla2019exploring} also verify causal confusion in realistic driving settings, and propose measures to address a specific instance of causal confusion. %

\begin{figure*}[t]
    \centering
    \begin{subfigure}[b]{0.32\textwidth}
        \includegraphics[width=\textwidth]{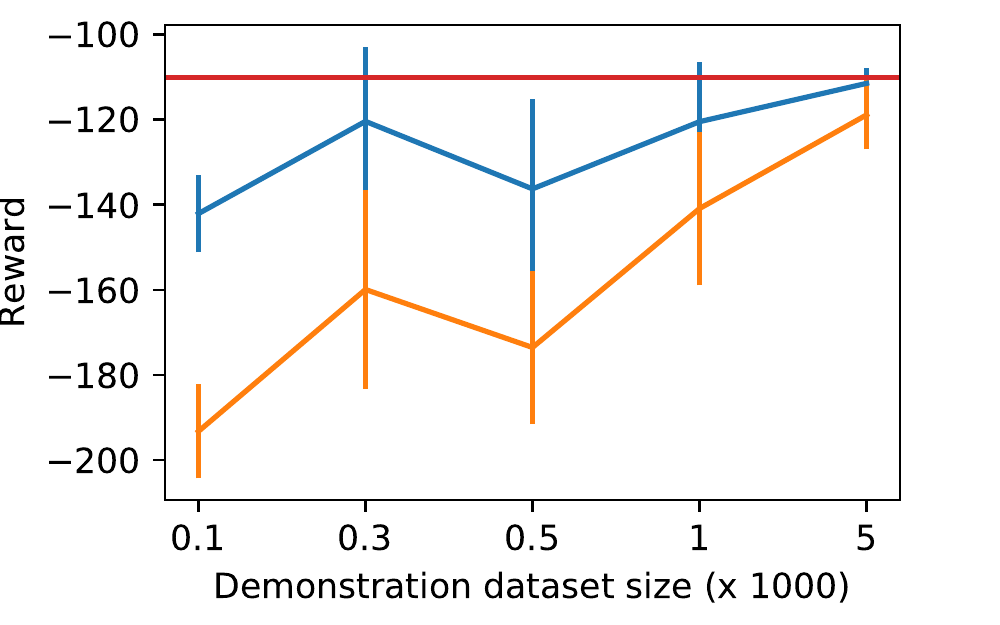}
        \caption{MountainCar}
    \end{subfigure}
    \begin{subfigure}[b]{0.32\textwidth}
        \includegraphics[width=\textwidth]{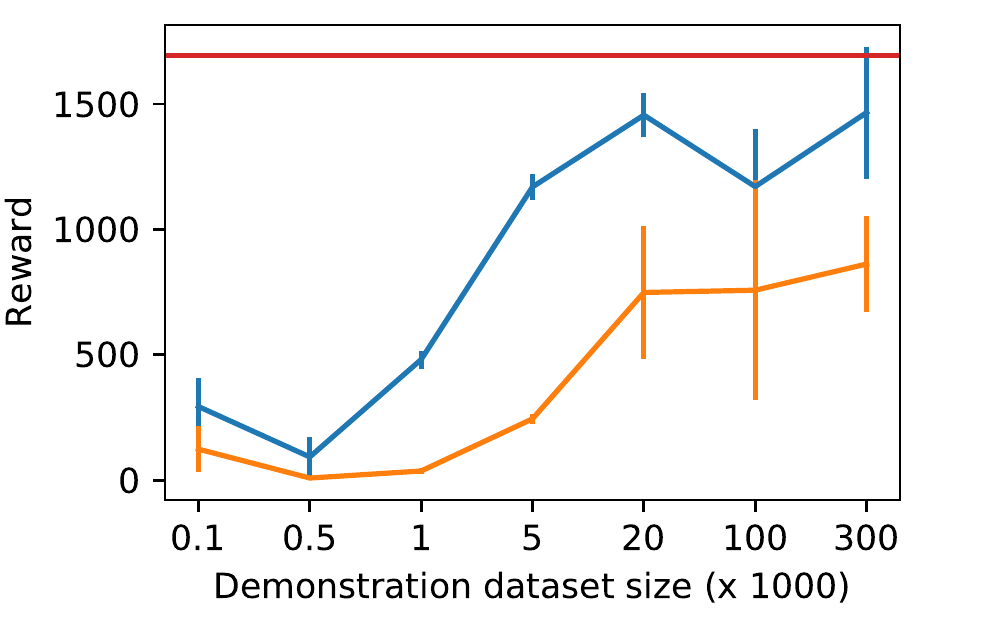}
        \caption{Hopper}
    \end{subfigure}
    \begin{subfigure}[b]{0.32\textwidth}
        \includegraphics[width=\textwidth]{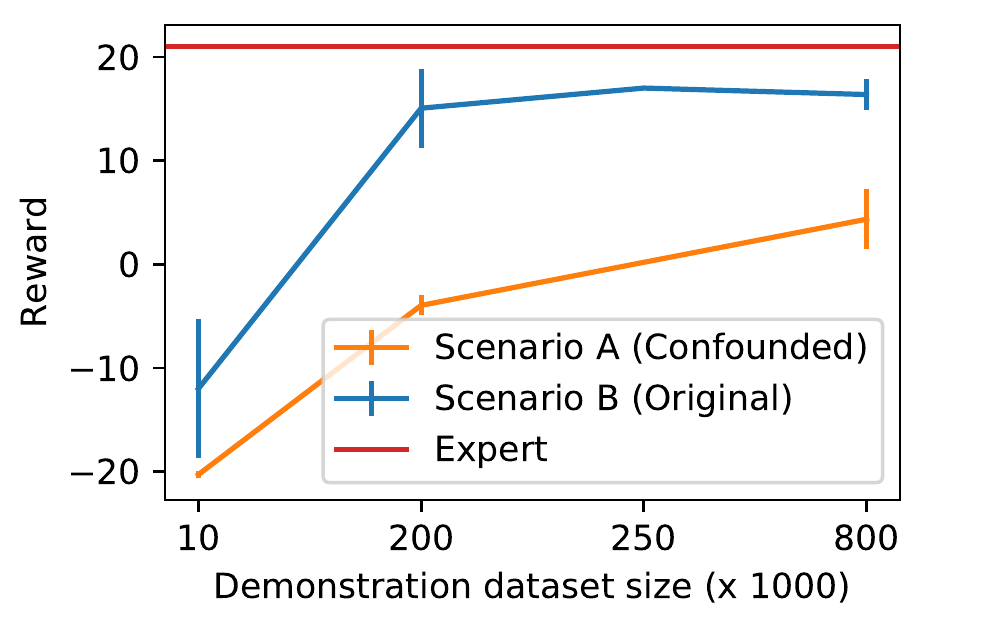}
        \caption{Pong}
    \end{subfigure}
    \caption{\small Diagnosing causal misidentification: net reward (y-axis) vs number of training samples (x-axis) for \method{original} and \method{confounded}, compared to expert reward (mean and stdev over 5 runs). Also see Appendix~\ref{app:diagnosis}.}    \label{fig:diagnosis}
    \vspace{-0.15in}
\end{figure*}

\vspace{-0.1in}
\section{Resolving Causal Misidentification}\label{sec:solution_framework}

Recall from Sec~\ref{sec:causal_dist_shift} that robustness to causal misidentification can be achieved by finding the true causal model of the expert's actions. We propose a simple pipeline to do this. First, we jointly learn policies corresponding to various causal graphs (Sec~\ref{sec:policies}). Then, we perform targeted interventions to efficiently search over the hypothesis set for the correct causal model (Sec~\ref{sec:intervention}).

\subsection{Causal Graph-Parameterized Policy Learning}\label{sec:policies}

 \begin{wrapfigure}{l}{0.3\textwidth}
   \vspace{-0.2in}
     \centering
     \includegraphics[width=0.3\textwidth]{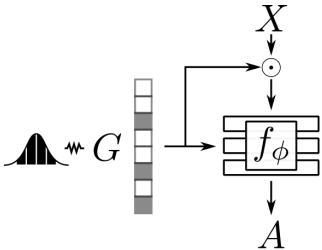}
     \caption{\small Graph-parameterized policy.} %
     \label{fig:arch}
 \end{wrapfigure}
In this step, we learn a policy corresponding to each candidate causal graph. Recall from Sec~\ref{sec:problem} that the expert's actions $A$ are based on an unknown subset of the state variables $\{X_i\}_{i=1}^n$. Each $X_i$ may either be a cause or not, so there are $2^n$ possible graphs.
We parameterize the structure $G$ of the causal graph as a vector of $n$ binary variables, each indicating the presence of an arrow from $X_k$ to $A$ in Fig~\ref{fig:graph}.
We then train a single graph-parameterized policy $\pi_G(X) = f_\phi([X \odot G, G])$, where $\odot$ is element-wise multiplication, and $[\cdot,\cdot]$ denotes concatenation. $\phi$ are neural network parameters, trained through gradient descent to minimize:
\begin{align}
    \mathbb{E}_G [\ell(f_\phi([X_i \odot G, G]), A_i)],
    \label{eq:main_loss}
\end{align}
where $G$ is drawn uniformly at random over all $2^n$ graphs and $\ell$ is a mean squared error loss for the continuous action environments and a cross-entropy loss for the discrete action environments. Fig~\ref{fig:arch} shows a schematic of the training time architecture. The policy network $f_\phi$ mapping observations $X$ to actions $A$ represents a mixture of policies, one corresponding to each value of the binary causal graph structure variable $G$, which is sampled as a bernoulli random vector.

In Appendix~\ref{app:variational-causal-discovery}, we propose an approach to perform variational Bayesian causal discovery over graphs $G$, using a latent variable model to infer a distribution over functional causal models (graphs and associated parameters)---the modes of this distribution are the FCMs most consistent with the demonstration data. This resembles the scheme above, except that instead of uniform sampling, graphs are sampled preferentially from FCMs that fit the training demonstrations well. We compare both approaches in Sec~\ref{sec:exp}, finding that simple uniform sampling nearly always suffices in preparation for the next step: targeted intervention.

\vspace{-0.1in}
\subsection{Targeted Intervention}\label{sec:intervention}
Having learned the graph-parameterized policy as in Sec~\ref{sec:policies}, we propose targeted intervention to compute the likelihood $\mathcal{L}(G)$ of each causal graph structure hypothesis $G$. In a sense, imitation learning provides an ideal setting for studying interventional causal learning: causal misidentification presents a clear challenge, while the fact that the problem is situated in a sequential decision process where the agent can interact with the world provides a natural mechanism for carrying out limited interventions.

We propose two intervention modes, both of which can be carried out by interaction with the environment via the actions:
\begin{wrapfigure}[29]{r}{0.51\textwidth}
\vspace{-0.15in}
\scalebox{0.9}{
\begin{minipage}{0.566\textwidth}
\begin{algorithm}[H]
  \caption{Expert query intervention}
  \label{alg:intervention-agreement}
\begin{algorithmic}
  \STATE {\bfseries Input:} policy network $f_\phi$ s.t.~$\pi_G(X)=f_\phi([X\odot G, G])$
  \STATE Initialize $w = 0, \mathcal{D} = \emptyset$.
  \STATE Collect states $\mathcal{S}$ by executing $\pi_{mix}$, the mixture of policies $\pi_G$ for uniform samples $G$.
  \STATE For each $X$ in $S$, compute disagreement score: \\ ~\phantom . \phantom . \phantom . $D(X)=\mathbb{E}_G[D_{KL}(\pi_G(X), \pi_{mix}(X))]$
  \STATE Select $\mathcal{S}' \subset \mathcal{S}$ with maximal $D(X)$.
  \STATE Collect state-action pairs $\mathcal{T}$ by querying expert on $\mathcal{S}'$.
  \FOR{$i=1 \dots N$}
  \STATE Sample $G \sim p(G) \propto \exp \langle w, G \rangle$.
  \STATE $\mathcal{L} \leftarrow \mathbb{E}_{s,a \sim \mathcal{T}}[\ell(\pi_G(s), a)]$
  \STATE $\mathcal{D} \leftarrow \mathcal{D} \cup \{(G, \mathcal{L}) \}$
  \STATE Fit $w$ on $\mathcal{D}$ with linear regression.
  \ENDFOR
  \STATE {\bfseries Return:} $\argmax_G p(G)$
\end{algorithmic}
\end{algorithm}
\vspace{-0.2in}
 \begin{algorithm}[H]
   \caption{Policy execution intervention}
   \label{alg:intervention-policy}
\begin{algorithmic}
   \STATE {\bfseries Input:} policy network $f_\phi$ s.t.~$\pi_G(X)=f_\phi([X\odot G, G])$
   \STATE Initialize $w = 0, \mathcal{D} = \emptyset$.
   \FOR{$i=1 \dots N$}
   \STATE Sample $G \sim p(G) \propto \exp \langle w, G \rangle$.
   \STATE Collect episode return $R_G$ by executing $\pi_G$.
   \STATE $\mathcal{D} \leftarrow \mathcal{D} \cup \{(G, R_G) \}$
   \STATE Fit $w$ on $\mathcal{D}$ with linear regression.
   \ENDFOR
   \STATE {\bfseries Return:} $\argmax_G p(G)$
\end{algorithmic}
\end{algorithm}
\vspace{-0.2in}
\end{minipage}
}
\end{wrapfigure}
\noindent \textbf{Expert query mode.}~~~This is the standard intervention approach applied to imitation learning: intervene on $X^t$ to assign it a value, and observe the expert response $A$. 
To do this, we sample a graph $G$ at the beginning of each intervention episode and execute the policy $\pi_G$. Once data is collected in this manner, we elicit expert labels on interesting states. %
    This requires an interactive expert, as in DAgger~\cite{ross2010efficient}, but requires substantially fewer expert queries than DAgger, because: (i) the queries serve only to disambiguate among a relatively small set of valid FCMs, and (ii) we use disagreement among the mixture of policies in $f_\phi$ to query the expert efficiently in an active learning approach. We summarize this approach in Algorithm~\ref{alg:intervention-agreement}.

    \noindent \textbf{Policy execution mode.}~~~It is not always possible to query an expert. For example, for a learner learning to drive a car by watching a human driver, it may not be possible to put the human driver into dangerous scenarios that the learner might
    encounter at intermediate stages of training. In cases like these where we would like to learn from
    pre-recorded demonstrations alone, we propose to intervene indirectly by using environmental returns (sum of rewards over time in an episode)
    $R=\sum_t r_t$. The policies $\pi_G(\cdot)=f_\phi([\cdot \odot G, G])$ corresponding to different hypotheses $G$ are executed in the environment and the returns $R_G$ collected. The likelihood of each graph is proportional to the exponentiated returns $\exp R_G$. The intuition is simple: environmental returns contain information about optimal expert policies even when experts are not queryable. Note that we do not even assume access to per-timestep rewards as in standard reinforcement learning; just the \emph{sum} of rewards for each completed run.
    As such, this intervention mode is much more flexible. See Algorithm~\ref{alg:intervention-policy}.

Note that both of the above intervention approaches evaluate individual hypotheses in isolation, but the number of hypotheses grows exponentially in the number of state variables. To handle larger states, we infer a graph distribution $p(G)$, by assuming an energy based model with a linear energy $E(G)=\langle w, G\rangle +b$, so the graph distribution is $p(G)=\prod_i p(G_i)=\prod_i \text{Bernoulli}(G_i|\sigma(w_i/\tau))$, where $\sigma$ is the sigmoid, which factorizes in independent factors. The independence assumption is sensible as our approach collapses $p(G)$ to its mode before returning it and the collapsed distribution is always independent.
$E(G)$ is inferred from linear regression on the likelihoods. This process is depicted in Algorithms~\ref{alg:intervention-agreement} and~\ref{alg:intervention-policy}.
The above method can be formalized within the reinforcement learning framework~\cite{levine2018probrl}. As we show in Appendix~\ref{app:soft-q}, the energy-based model can be seen as an instance of soft Q-learning~\cite{haarnoja2017reinforcement}.

\vspace{-0.1 in}
\subsection{Disentangling Observations}~\label{sec:vae}
In the above, we have assumed access to disentangled observations $X^t$. When this is not the case, such as with image observations, $X^t$ must be set to a disentangled representation of the observation at time $t$. We construct such a representation by training a $\beta$-VAE~\cite{kingma2013auto,higgins2016beta} to reconstruct the original observations. %
To capture states beyond those encountered by the expert, we train with a mix of expert and random trajectory states. Once trained, $X^t$ is set to be the mean of the latent distribution produced at the output of the encoder. The VAE training objective encourages disentangled dimensions in the latent space~\cite{burgess2018understanding,chen2018isolating}. We employ CoordConv~\cite{coordconv} in both the encoder and the decoder architectures.

\vspace{-0.1in}
\section{Experiments}\label{sec:exp}
\vspace{-0.1in}

We now evaluate the solution described in Sec~\ref{sec:solution_framework} on the five tasks (MountainCar, Hopper, and 3 Atari games) described in Sec~\ref{sec:benchmark}. In particular, recall that \method{confounded} performed significantly worse than \method{original} across all tasks. In our experiments, we seek to answer the following questions: \textbf{(1)} Does our targeted intervention-based solution to causal misidentification bridge the gap between \method{confounded} and \method{original}? \textbf{(2)} How quickly does performance improve with intervention? \textbf{(3)} Do both intervention modes (expert query, policy execution) described in Sec~\ref{sec:intervention} resolve causal misidentification? \textbf{(4)} Does our approach in fact recover the true causal graph? \textbf{(5)} Are disentangled state representations necessary?

In each of the two intervention modes,
we compare two variants of our method: \method{unif-intervention} and \method{disc-intervention}. They only differ in the training of the graph-parameterized mixture-of-policies $f_\phi$---while \method{unif-intervention} samples causal graphs uniformly, %
\method{disc-intervention} uses the variational causal discovery approach mentioned in Sec~\ref{sec:policies}, and described in detail in Appendix~\ref{app:variational-causal-discovery}.

\begin{figure}[t]
    \centering
    \includegraphics[trim={0.2cm 0.2cm 0.0cm 0.0cm}, clip,height=8em]{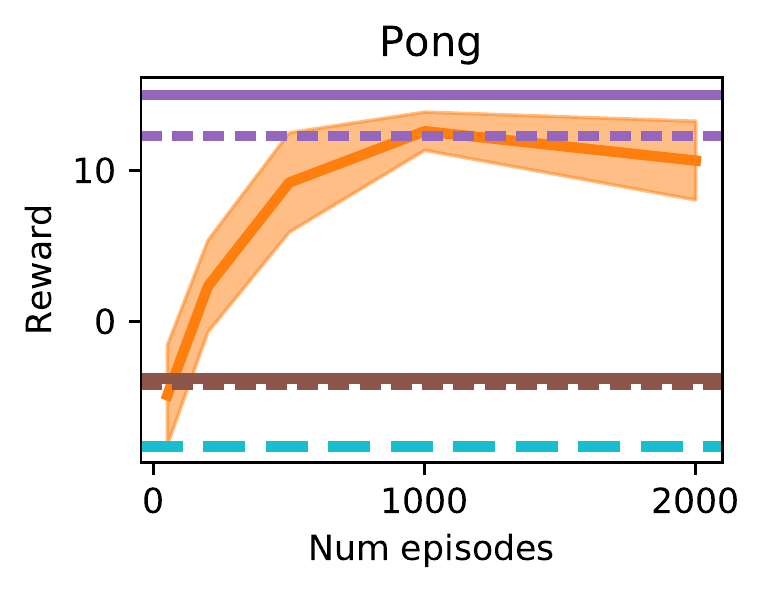}
    \includegraphics[trim={0.2cm 0.2cm 0.0cm 0.0cm}, clip,height=8em]{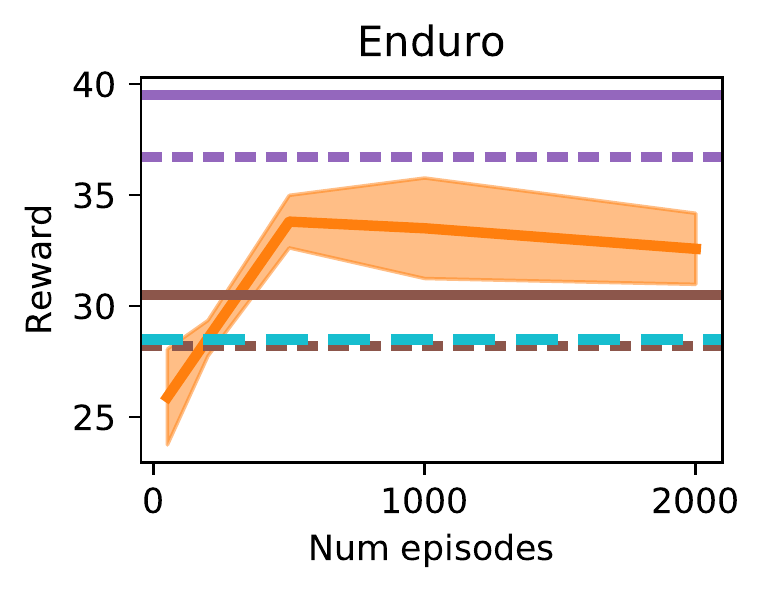}
    \includegraphics[trim={0.2cm 0.2cm 0.0cm 0.0cm}, clip,height=8em]{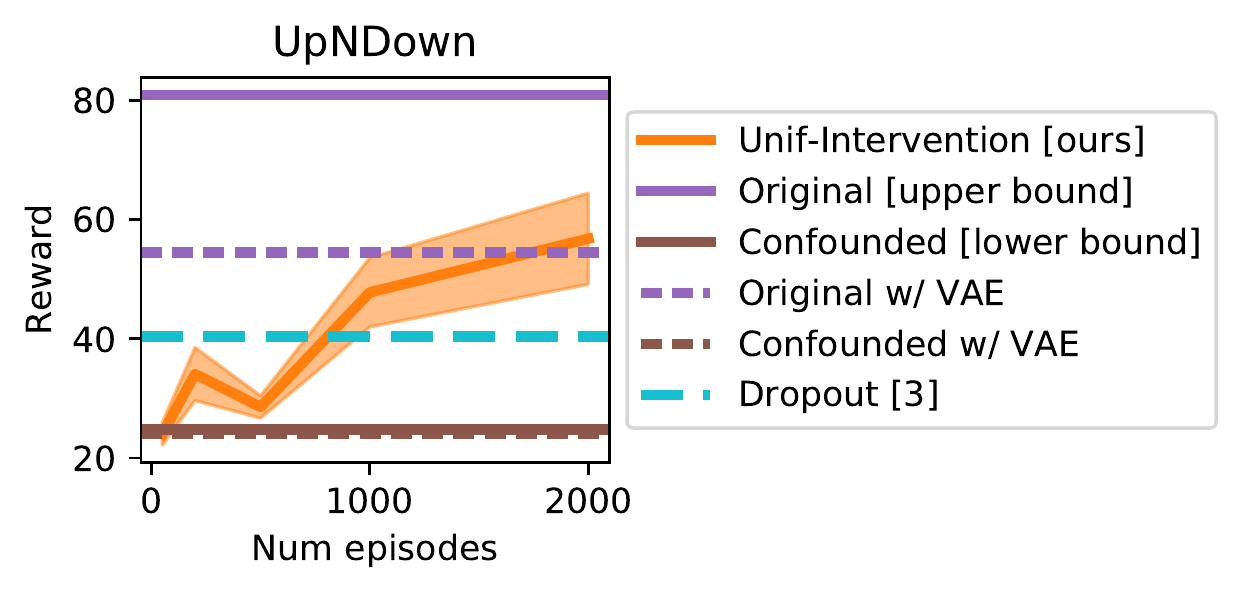}
    \caption{\small Reward vs.~number of intervention episodes (policy execution interventions) on Atari games. \method{unif-intervention} succeeds in getting rewards close to \method{original w/ vae}, while the \method{dropout} baseline only outperforms \method{confounded w/ vae} in UpNDown.}
    \label{fig:atari-rl}
    \vspace{-0.2in}
\end{figure}
\noindent \textbf{Baselines.}~~~We compare our method against three baselines applied to the confounded state. \method{dropout} trains the policy using Eq~\ref{eq:main_loss} and evaluates with the graph $G$ containing all ones, which amounts to dropout regularization~\cite{dropout} during training, as proposed by~\citet{chauffeurnet}.  \method{dagger}~\cite{ross2010efficient} addresses distributional shift by querying the expert on states encountered by the imitator, requiring an interactive expert. We compare \method{dagger} to our expert query intervention approach. Lastly, we compare to Generative Adversarial Imitation Learning (\method{gail})~\cite{ho2016generative}. \method{gail} is an alternative to standard behavioral cloning that works by matching demonstration trajectories to those generated by the imitator during roll-outs in the environment. Note that the PC algorithm~\cite{le2016fast}, commonly used in causal discovery from passive observational data, relies on the faithfulness assumption, which causes it to be infeasible in our setting, as explained in Appendix~\ref{app:pc}. See Appendices \ref{app:prop} \& \ref{app:variational-causal-discovery} for details.%

\begin{figure}[b]
    \vspace{-0.25in}
    \centering
    \includegraphics[trim={.2cm 0.3cm 1cm 0.0cm}, clip,height=10em]{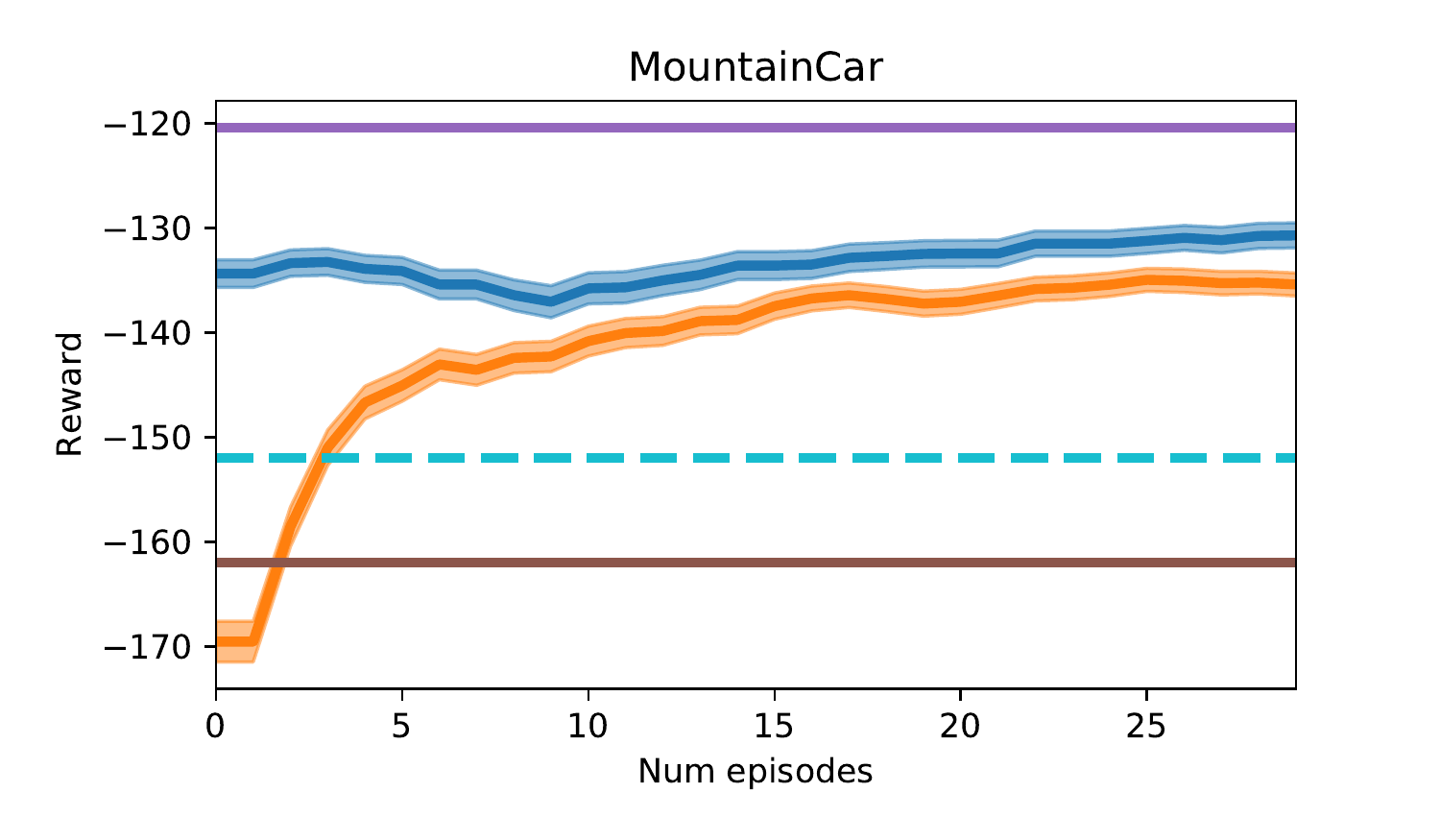}
    \includegraphics[trim={.2cm 0.3cm 0cm 0.0cm}, clip,height=10em]{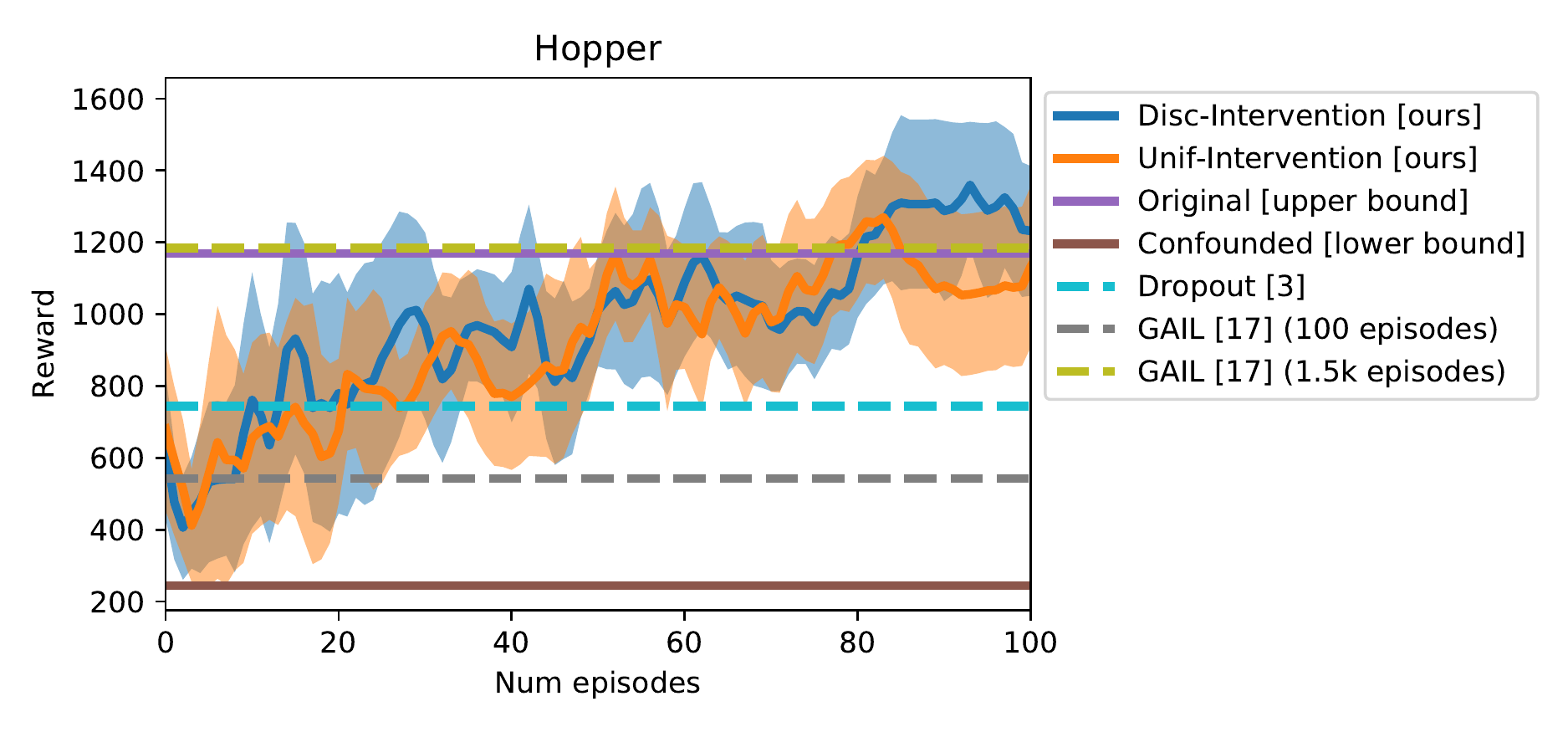}
    \caption{\small Reward vs.~number of intervention episodes (policy execution interventions) on MountainCar and Hopper. Our methods \method{unif-intervention} and \method{disc-intervention} bridge most of the causal misidentification gap (between \method{original} (lower bound) and \method{confounded} (upper bound), approaching \method{original} performance after tens of episodes. \method{gail}~\cite{ho2016generative} (on Hopper) achieves this too, but after 1.5k episodes.}
    \label{fig:rollout_intervention_posterior}
\end{figure}

\noindent \textbf{Intervention by policy execution.}~~~Fig~\ref{fig:rollout_intervention_posterior} plots episode rewards versus number of policy execution intervention episodes for MountainCar and Hopper. The reward always corresponds to the current mode $\argmax_G p(G)$ of the posterior distribution over graphs, updated after each episode, as described in Algorithm~\ref{alg:intervention-policy}. In these cases, both \method{unif-intervention} and \method{disc-intervention} eventually converge to models yielding similar rewards, which we verified to be the correct causal model i.e., true causes are selected and nuisance correlates left out. %
In early episodes on MountainCar, \method{disc-intervention} benefits from the prior over graphs inferred in the variational causal discovery phase. However, in Hopper, the simpler \method{unif-intervention} performs just as well. \method{dropout} does indeed help in both settings, as reported in~\citet{chauffeurnet}, but is significantly poorer than our approach variants. \method{gail} requires about 1.5k episodes on Hopper to match the performance of our approaches, which only need tens of episodes. Appendix~\ref{app:gail} further analyzes the performance of \method{gail}. Standard implementations of \method{gail} do not handle discrete action spaces, so we do not evaluate it on MountainCar.

\begin{figure}%
    \centering
    \includegraphics[trim={0.2cm 0.3cm 1cm 0.0cm}, clip,height=10em]{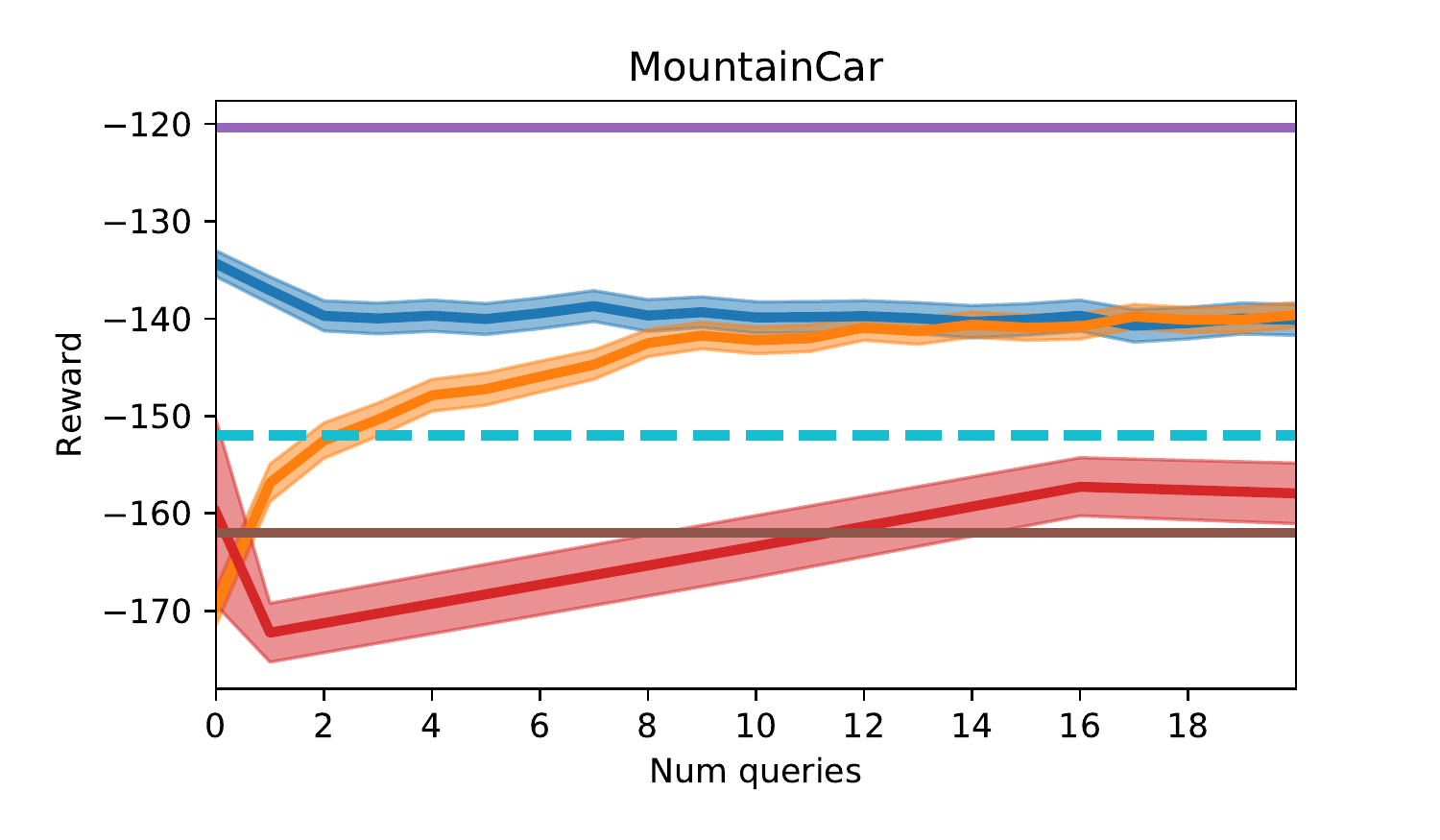}
    \includegraphics[trim={0.2cm 0.3cm 0cm 0.0cm}, clip,height=10em]{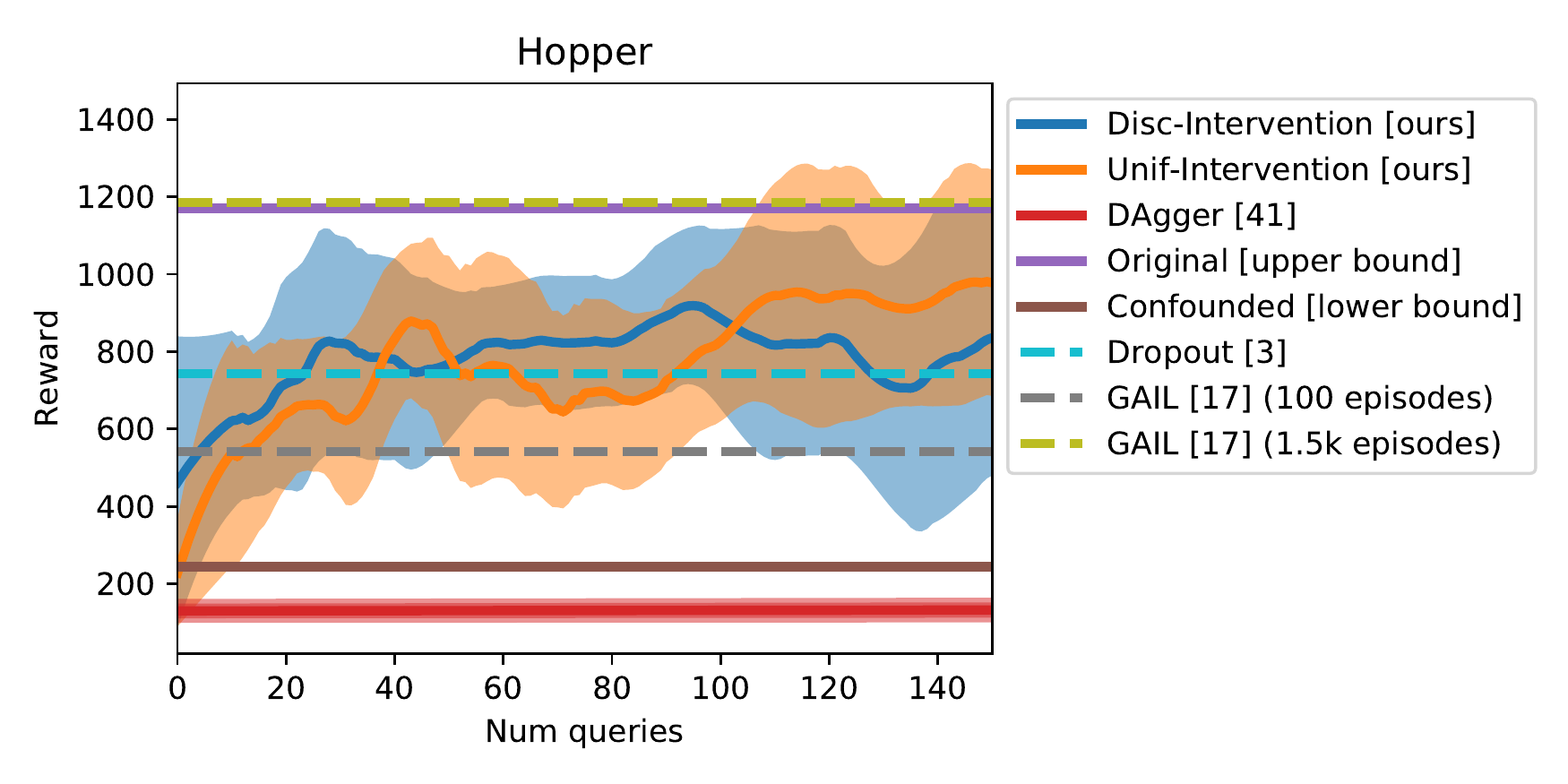}
    \caption{\small Reward vs.~expert queries (expert query interventions) on MountainCar and Hopper. Our methods partially bridge the gap from \method{confounded} (lower bd) to \method{original} (upper bd), also outperforming \method{dagger}~\cite{ross2011reduction} and \method{dropout}~\cite{chauffeurnet}. \method{gail}~\cite{ho2016generative} outperforms our methods on Hopper, but requires a large number of policy roll-outs (also see Fig~\ref{fig:rollout_intervention_posterior} comparing \method{gail} to our policy execution-based approach).} %
    \label{fig:expert_intervention_posterior}
    \vspace{-0.2in}
\end{figure}

As described in Sec~\ref{sec:vae}, we use a VAE to disentangle image states in Atari games to produce 30-D representations for Pong and Enduro and 50-D representations for UpNDown. We set this dimensionality heuristically to be as small as possible, while still producing good reconstructions as assessed visually. Requiring the policy to utilize the VAE representation without end-to-end training does result in some drop in performance, as seen in Fig~\ref{fig:atari-rl}. However, causal misidentification still causes a very large drop of performance even relative to the baseline VAE performance. \method{disc-intervention} is hard to train as the cardinality of the state increases, and yields only minor advantages on Hopper (14-D states), so we omit it for these Atari experiments.
As Fig~\ref{fig:atari-rl} shows, \method{unif-intervention} indeed  improves significantly over \method{confounded w/ vae} in all three cases, matching \method{original w/ vae} on Pong and UpNDown, while the \method{dropout} baseline only improves UpNDown. In our experiments thus far, \method{gail} fails to converge to above-chance performance on any of the Atari environments. These results show that our method successfully alleviates causal misidentification within relatively few trials.

\noindent \textbf{Intervention by expert queries.}~~~Next, we perform direct intervention by querying the expert on samples from trajectories produced by the different causal graphs. %
In this setting, we can also directly compare to \method{dagger}~\cite{ross2011reduction}.
Fig~\ref{fig:expert_intervention_posterior} shows results on MountainCar and Hopper. Both our approaches successfully improve over \method{confounded} within a small number of queries. Consistent with policy execution intervention results reported above, we verify that our approach again identifies the true causal model correctly in both tasks, and also performs better than \method{dropout} in both settings. It also exceeds the rewards achieved by \method{dagger}, while using far fewer expert queries. In Appendix~\ref{app:dagger}, we show that \method{dagger} requires hundreds of queries to achieve similar rewards for MountainCar and tens of thousands for Hopper. Finally, \method{gail} with 1.5k episodes outperforms our expert query interventions approach. Recall however from Fig~\ref{fig:expert_intervention_posterior} that this is an order of magnitude more than the number of episodes required by our policy intervention approach. %

Once again, \method{disc-intervention} only helps in early interventions on MountainCar, and not at all on Hopper. Thus, our method's performance is primarily attributable to the targeted intervention stage, and the exact choice of approach used to learn the mixture of policies is relatively insignificant. %

Overall, of the two intervention approaches, policy execution converges to better final rewards. Indeed, for the Atari environments, we observed that expert query interventions proved ineffective. We believe this is because expert agreement is an imperfect proxy for true environmental rewards. %
\textbf{Interpreting the learned causal graph.}
Our method labels each dimension of the VAE encoding of the frame as a cause or nuisance variable. In Fig~\ref{fig:attention}, we analyze these inferences in the Pong environment as follows: in the top row, a frame is encoded into the VAE latent, then for all nuisance dimensions (as inferred by our approach \method{unif-intervention}), that dimension is replaced with a sample from the prior, and new samples are generated. In the bottom row, the same procedure is applied with a random graph that has as many nuisance variables as the inferred graph.
We observe that in the top row, the causal variables (the ball and paddles) are shared between the samples, while the nuisance variables (the digit) differ, being replaced either with random digits or unreadable digits. In the bottom row, the causal variables differ strongly, indicating that important aspects of the state are judged as nuisance variables. This validates that, consistent with MountainCar and Hopper, our approach does indeed identify true causes in Pong.
\begin{figure}[tb]
    \centering
    \begin{subfigure}[b]{0.7\textwidth}
        \includegraphics[width=\textwidth]{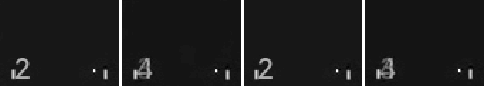}
    \end{subfigure}\\
    \vspace{0.05in}
    \begin{subfigure}[b]{0.7\textwidth}
        \includegraphics[width=\textwidth]{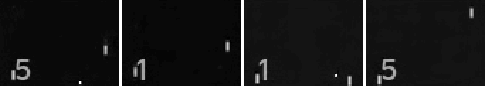}
    \end{subfigure}
    \caption{\small
    Samples from (top row) learned causal graph and (bottom row) random causal graph. (See text)
    }
    \label{fig:attention}
\end{figure}

\begin{wraptable}[8]{r}{0.5\textwidth}
    \vspace{-0.1in}
    \small
    \centering
    \begin{tabular}{ll|c}
     \toprule
        Mode & Representation & Reward  \\ \midrule
        Policy execution & Disentangled & \textbf{-137} \\
         & Entangled & -145 \\\midrule
        Expert queries & Disentangled & \textbf{-140} \\
         & Entangled & -165 \\
         
        \bottomrule
    \end{tabular}
    \caption{\small Intervention on (dis)entangled MountainCar.}
    \label{tab:entangled}
  
\end{wraptable}
\noindent \textbf{Necessity of disentanglement.}~~~
Our intervention method assumes a disentangled representation of state. Otherwise, each of the $n$ individual dimensions in the state might capture both causes as well as nuisance variables and the problem of discovering true causes is no longer reducible to searching over $2^n$ graphs.

To test this empirically, we create a variant of our MountainCar \method{confounded} testbed, where the 3-D past action-augmented state vector is rotated by a fixed, random rotation.
After training the graph-conditioned policies on the entangled and disentangled \method{confounded} state, and applying 30 episodes of policy execution intervention or 20 expert queries, we get the results shown in Tab~\ref{tab:entangled}. The results are significantly lower in the entangled than in the disentangled (non-rotated) setting, indicating disentanglement is important for the effectiveness of our approach.

\section{Conclusions}

We have identified a naturally occurring and fundamental problem in imitation learning, ``causal misidentification'', and proposed a causally motivated approach for resolving it. 
While we observe evidence for causal misidentification arising in natural imitation learning settings, we have thus far validated our solution in somewhat simpler synthetic settings intended to mimic them. %
Extending our solution to work for such realistic scenarios is an exciting direction for future work. Finally, %
apart from imitation, general machine learning systems deployed in the real world also encounter ``feedback''~\cite{ sculley2014machine, feedback_video}, which opens the door to causal misidentification. We hope to address these more general settings in the future.

\paragraph{Acknowledgments:} We would like to thank Karthikeyan Shanmugam and Shane Gu for pointers to prior work early in the project, and Yang Gao, Abhishek Gupta, Marvin Zhang, Alyosha Efros, and Roberto Calandra for helpful discussions in various stages of the project. We are also grateful to Drew Bagnell and Katerina Fragkiadaki for helpful feedback on an earlier draft of this paper. This project was supported in part by Berkeley DeepDrive, NVIDIA, and Google.

\bibliography{references}
\bibliographystyle{plainnat}
\clearpage

\appendix
\section{Expert Demonstrations}\label{app:demonstrations}
To collect demonstrations, we first train an expert with reinforcement learning. We use DQN \cite{dqn} for MountainCar,
TRPO \cite{schulman2015trust} for Hopper, and PPO~\cite{schulman2017proximal} for the Atari environments (Pong, UpNDown, Enduro). This expert policy is executed in the environment to collect demonstrations.

\section{Necessity of Correct Causal Model}\label{app:prop}

\textbf{Faithfulness:} A causal model is said to be faithful when all conditional independence relationships in the distribution are represented in the graph.

We pick up the notation used in Sec~\ref{sec:causal_dist_shift}, but for notational simplicity, we drop the time superscript for $X$, $A$, and $Z$ when we are not reasoning about multiple time-steps.

\begin{prop}
Let the expert's functional causal model be $(G^*, \theta_{G^*}^*)$, with causal graph $G^*\in \mathcal{G}$ as in Figure \ref{fig:graph} and function parameters $\theta_{G^*}^*$.
We assume some faithful learner $(\hat G, \theta_{\hat G}), \hat G\in \mathcal{G}$ that agrees on the interventional query:
$$\forall X, A: p_{G^*, \theta_{G^*}^*}(A|do(X)) = p_{\hat G, \theta_{\hat G}}(A|do(X))$$
Then it must be that $G^* = \hat G$.\footnote{We drop time $t$ from the superscript when discussing states and actions from the same time.}
\end{prop}
\begin{proof}
For graph $G$, define the index set of state variables that are independent of the action in the mutilated graph $G_{\bar{X}}$:
$$I_G = \{i | X_i \underset{G_{\bar{X}}}{\independent}  A\}$$
From the assumption of matching interventional queries and the assumption of faithfulness, it follows that: $I_{G^*}=I_{\hat G}$.
From the graph, we observe that $I_G = \{i | (X_i \to A) \not\in G \}$ and thus $G^*=\hat G$.
\end{proof}

\section{Passive Causal Discovery, Faithfulness and Determinism}\label{app:pc}
\begin{wraptable}{r}{0.5\textwidth}
    \vspace{-0.1in}
    \small
    \centering
    \begin{tabular}{l|cc}
     \toprule
        & $I(X^t_i ; A^t)$ & $I(X^t_i ; A^t | Z^t)$  \\ \midrule
        $X^t_0$ (cause) & 0.377 &   0.013 \\
        $X^t_1$ (cause) & 0.707 &   0.019 \\
        $X^t_2$ (nuisance) & 0.654 &   0.027 \\
        \bottomrule
    \end{tabular}
    \caption{\small Mutual information in bits of the \method{confounded} MountainCar setup.}
    \label{tab:passive-discovery}
  
\end{wraptable}
In many learning scenarios, much information about the causal model can already be inferred passively from the data. This is the problem of causal discovery. Ideally, it would allow us to perform statistical analysis on the random variables in Fig~\ref{fig:graph} in the demonstration data to determine whether variable $X^t_i$ is a cause of the next expert action $A^t$ or a nuisance variable.

Causal discovery algorithms, such as the PC algorithm~\cite{spirtes2000causation} test a series of conditional independence relationships in the observed data and construct the set of possible causal graphs whose conditional independence relationships match the data. It does so by assuming \emph{faithfulness}, meaning the joint probability of random variables contains no more conditional independence relationships than the causal graph. In the particular case of the causal model in Fig~\ref{fig:graph}, it is easy to see that $X^t_i$ is a cause of $A^t$, and thus that the arrow $X^t_i \to A^t$ exists, if and only if $X^t_i \not\independent A^t | Z^t$, meaning that $X^t_i$ provides extra information about $A^t$ if $Z^t$ is already known.

We test this procedure empirically by evaluating the mutual information $I(X^t_i ; A^t | Z^t)$ for the \method{confounded} MountainCar benchmark, using the estimator from \citet{gao2017estimating}. The results in Table~\ref{tab:passive-discovery} show that all state variables are correlated with the expert's action, but that all become mostly independent given the confounder $Z^t$, implying none are causes.

Passive causal discovery failed because the critical faithfulness assumption is violated in the MountainCar case. Whenever a state variable $X^t_i$ is a deterministic function of the past $Z^t$, so that $X^t_i \independent A^t | Z^t$ always holds and a passive discovery algorithm concludes no arrow $X^t_i \to A^t$ exists. Such a deterministic transition function for at least a part of the state is very common in realistic imitation learning scenarios, making passive causal discovery inapplicable. Active interventions must thus be used to determine the causal model.

\section{Variational Causal Discovery}\label{app:variational-causal-discovery}
\begin{figure}[ht]
    \centering
    \vspace{-0.2in}
    \includegraphics[width=0.7\linewidth]{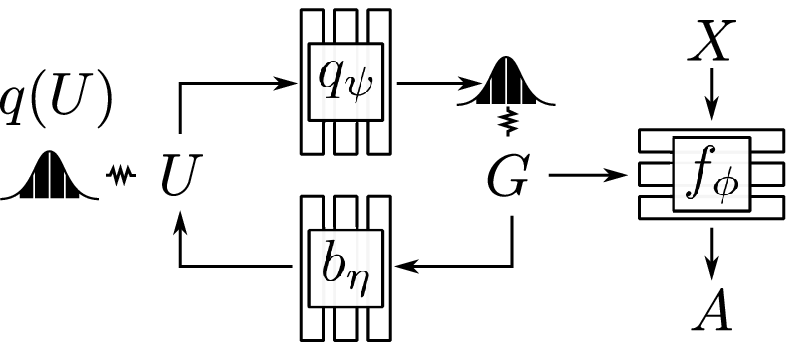}
    \caption{\small Training architecture for variational inference-based causal discovery as described in Appendix~\ref{app:variational-causal-discovery}. The policy network $f_\phi$ represents a mixture of policies, one corresponding to each value of the binary causal graph structure variable $G$. This variable in turn is sampled from the distribution $q_\psi(G|u)$ produced by an inference network from an input latent $U$. Further, a network $b_\eta$ regresses back to the latent $U$ to enforce that $G$ should not be independent of $U$.}
    \label{fig:arch_variational}
\end{figure}

The problem of discovering causal graphs from passively observed data is called causal discovery. The PC algorithm~\cite{spirtes2000causation} is arguably the most widely used and easily implementable causal discovery algorithm. In the case of Fig~\ref{fig:graph}, the PC algorithm would imply the absence of the arrow $X^t_i \rightarrow A^t$, if the conditional independence relation $A^t \independent X^t_i | Z^t$ holds,
which can be tested by measuring the mutual information.
However, the PC algorithm relies on \emph{faithfulness} of the causal graph. That is, conditional independence must imply d-separation in the graph. However, faithfulness is easily violated in a Markov decision process. If for some $i$, $X^t_i$ is a cause of the expert's action $A^t$ (the arrow $X^t_i \rightarrow A^t$ should exist), but $X^t_i$ is the result of a deterministic function of $Z^t$, then always $A^t \independent X^t_i | Z^t$ and the PC algorithm would wrongly conclude that the arrow $X^t_i \rightarrow A^t$ is absent.\footnote{More generally, faithfulness places strong constraints on the expert graph. For example, a visual state may contain unchanging elements such as the car frame in Fig~\ref{fig:causal_misidentification}, which are by definition deterministic functions of the past. As another example, goal-conditioned tasks must include a constant goal in the state variable at each time, which once again has deterministic transitions, violating faithfulness.}

We take a Bayesian approach to causal discovery~\citep{Heckerman2006bayesiancausal} from demonstrations. Recall from Sec~\ref{sec:problem} that the expert's actions $A$ are based on an unknown subset of the state variables $\{X_i\}_{i=1}^n$. Each $X_i$ may either be a cause or not, so there are $2^n$ possible graphs.
We now define a variational inference approach to infer a distribution over functional causal models (graphs and associated parameters) such that its modes are consistent with the demonstration data $D$.

While Bayesian inference is intractable, variational inference can be used to find a distribution that is close to the true posterior distribution over models.
We parameterize the structure $G$ of the causal graph as %
a vector of $n$ correlated Bernoulli random variables $G_k$, each indicating the presence of a causal arrow from $X_k$ to $A$. We assume a variational family with a point estimate $\theta_G$ of the parameters corresponding to graph $G$ and use a latent variable model to describe the correlated Bernoulli variables, with a standard normal distribution $q(U)$ over latent random variable $U$:
\begin{align*}
 q_\psi(G, \theta) &= q_\psi(G)[\theta = \theta_G] \\
 &= \int q(U) \prod_{k=1}^n q_\psi(G_k | U) [\theta = \theta_G] dU
 \label{eq:variational-family}
\end{align*}
We now optimise the evidence lower bound (ELBO):
\begin{align}
&\nonumber\argmin_q D_{KL}(q_\psi(G, \theta) | p(G, \theta | D)) = \\
&\argmax_{\psi,\theta} \sum_i \mathbb{E}_{U \sim q(U), G_k \sim q_\psi(G_k|U)} \\
&\quad\left[\log \pi(A_i | X_i, G, \theta_G) + \log p(G) \right]  + \mathcal{H}_q(G)
\label{eq:loss}
\end{align}

\paragraph{Likelihood}
$ \pi(A_i | X_i, G, \theta_G)$ is the likelihood of the observations $X$ under the FCM  $(G, \theta_G)$. It is modelled by a single neural network $f_\phi([X \odot G, G])$, where $\odot$ is the element-wise multiplication, $[\cdot,\cdot]$ denotes concatenation and $\phi$ are neural network parameters.
\paragraph{Entropy} The entropy term of the KL divergence, $\mathcal{H}_q$, acts as a regularizer to prevent the graph distribution from collapsing to the maximum a-posteriori estimate. It is intractable to directly maximize entropy, but a tractable variational lower bound can be formulated. Using the product rule of entropies, we may write:
\begin{align*}
\mathcal{H}_q(G) &= \mathcal{H}_q(G|U) - \mathcal{H}_q(U|G) + \mathcal{H}_q(U) \\
&=  \mathcal{H}_q(G|U) + I_q(U; G)
\end{align*}
In this expression, $\mathcal{H}_q(G|U)$ promotes diversity of graphs, while $I_q(U; G)$ encourages correlation among $\{G_k\}$. $I_q(U; G)$ can be bounded below using the same variational bound used in InfoGAN~\cite{chen2016infogan}, with a variational distribution $b_\eta$:
$I_q(U; G) \ge \mathbb{E}_{U,G \sim q_\psi}\log b_\eta(U|G) $.
Thus, during optimization, in lieu of entropy, we maximize the following lower bound:
$$\mathcal{H}_q(G) \ge \mathbb{E}_{U,G \sim q}\left[-\sum_k \log q_\psi(G_k|U) + \log b_\eta(U|G)\right]$$

\paragraph{Prior}
The prior $p(G)$ over graph structures is set to prefer graphs with fewer causes for action $A$---it is thus a sparsity prior: $$p(G) \propto \exp \sum_k [G_k=1]$$
\paragraph{Optimization}
Note that $G$ is a discrete variable, so we cannot use the reparameterization trick~\citep{kingma2013auto}. Instead, we use the Gumbel Softmax trick~\citep{gumbelsoftmax,concrete}
to compute gradients for training $q_\psi(G_k | U)$. Note that this does not affect $f_\phi$, which can be trained with standard backpropagation.

The loss of Eq~\ref{eq:loss} is easily interpretable independent of the formalism of variational Bayesian causal discovery. A mixture of predictors $f_\phi$ is jointly trained, each paying attention to diverse sparse subsets (identified by $G$) of the inputs. This is related to variational dropout~\citep{kingma2015dropout}.
Once this model is trained, $q_\psi(G)$ represents the hypothesis distribution over graphs, and $\pi_G(x)=f_\phi([x\odot G, G])$ represents the imitation policy corresponding to a graph $G$. Fig~\ref{fig:arch_variational} shows the architecture.

\paragraph{Usage for Targeted Interventions} In our experiments, we also evaluate the usefulness of causal discovery process to set a prior for the targeted interventions described in Sec~\ref{sec:intervention}. In Algorithm~\ref{alg:intervention-agreement} and ~\ref{alg:intervention-policy}, we implement this by initializing $p(G)$ to the discovered distribution (rather than uniform).

\section{Additional Results: Diagnosing Causal Misidentification}\label{app:diagnosis}
\begin{figure*}[htb]
    \centering
    \begin{subfigure}[b]{\textwidth}
        \includegraphics[width=0.9\textwidth]{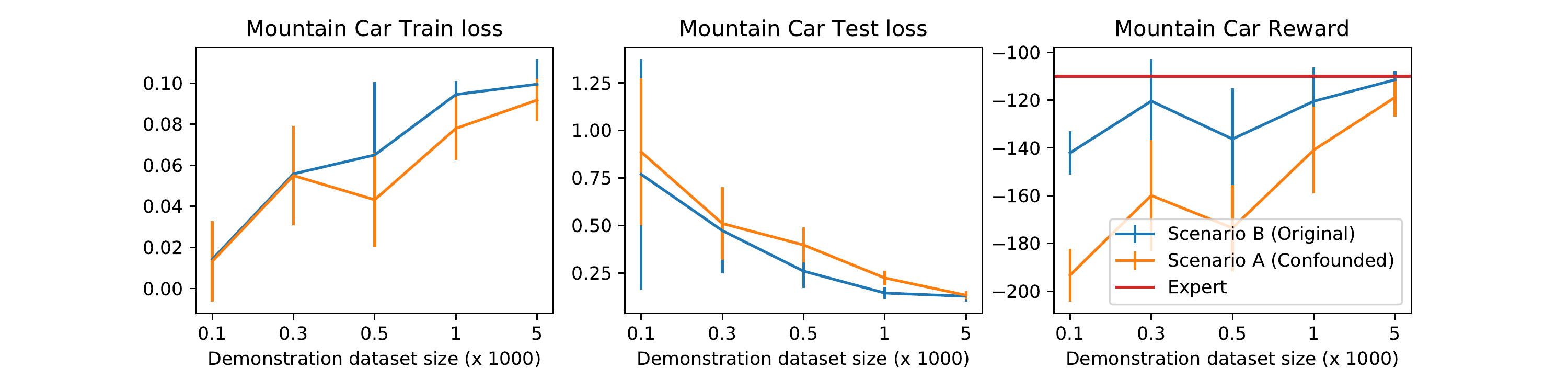}
        \label{fig:diagnosis-mountaincar}
    \end{subfigure}
    \begin{subfigure}[b]{\textwidth}
        \includegraphics[width=0.9\textwidth]{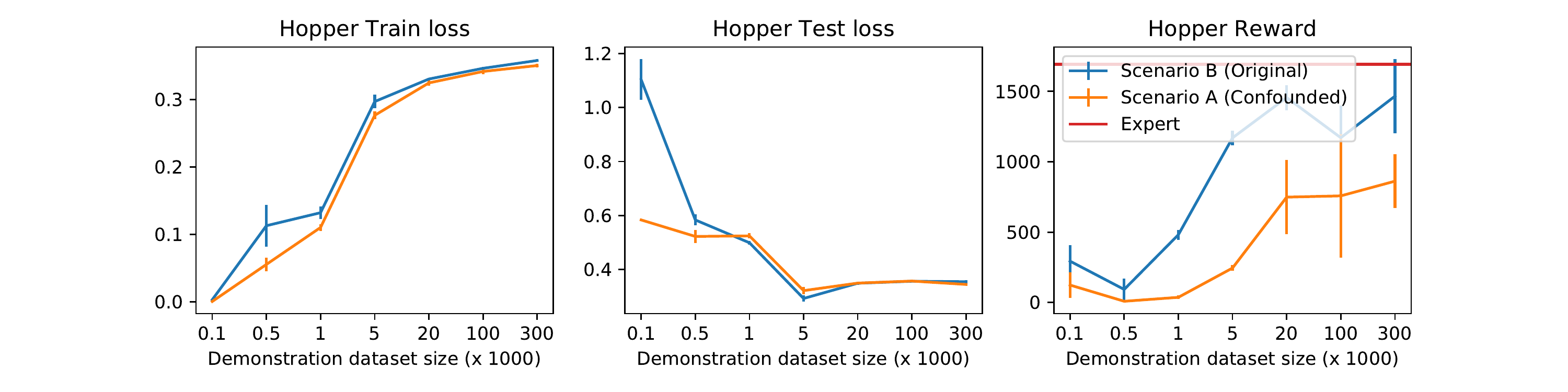}
        \label{fig:diagnosis-hopper}
    \end{subfigure}
    \begin{subfigure}[b]{\textwidth}
        \includegraphics[width=0.9\textwidth]{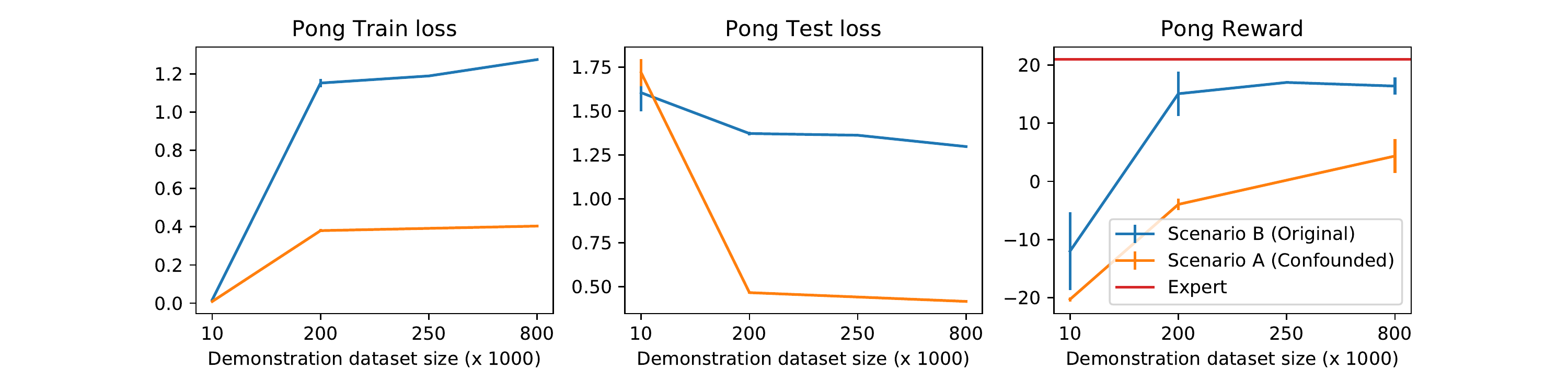}
        \label{fig:diagnosis-pong}
    \end{subfigure}
    \begin{subfigure}[b]{\textwidth}
        \includegraphics[width=0.9\textwidth]{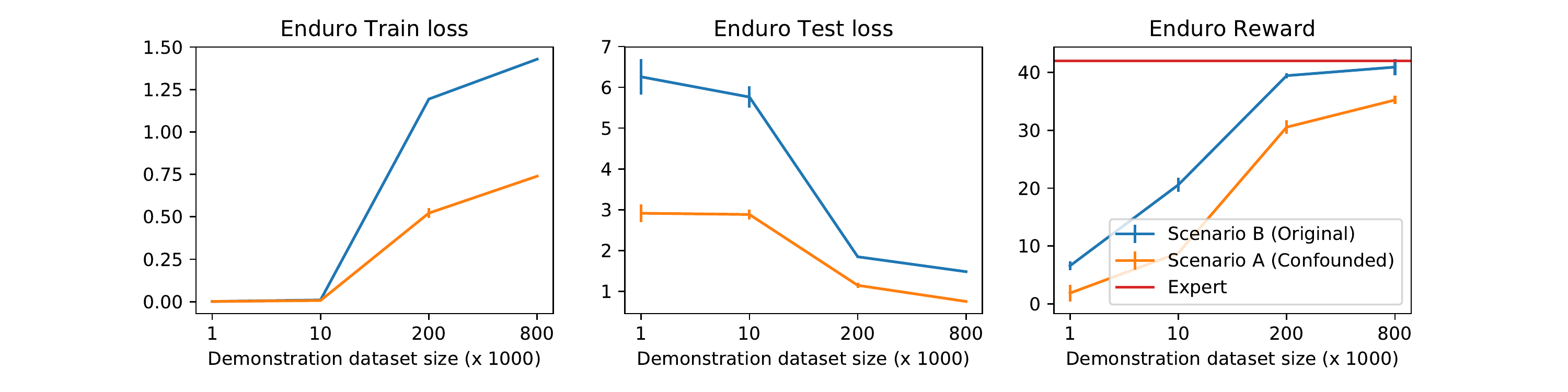}
        \label{fig:diagnosis-enduro}
    \end{subfigure}
    \begin{subfigure}[b]{\textwidth}
        \includegraphics[width=0.9\textwidth]{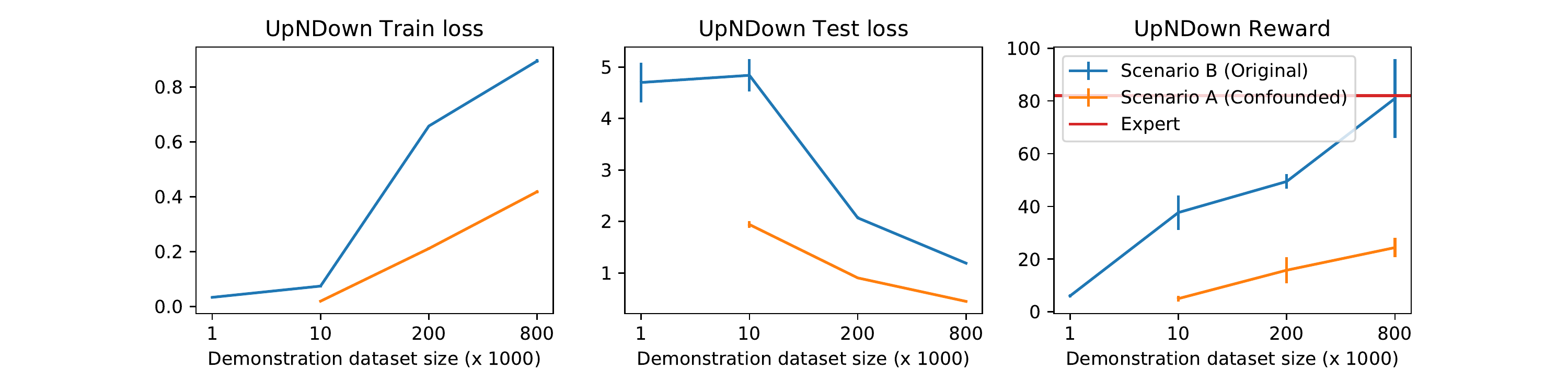}
        \label{fig:diagnosis-upndown}
    \end{subfigure}
    \caption{An expanded version of Fig~\ref{fig:diagnosis} in the main paper, demonstrating diagnosis of the causal misidentification problem in three settings. Here, the final reward, shown in Fig~\ref{fig:diagnosis} is shown in the third column. Additionally, we  also show the behavior cloning training loss (first column) and validation loss (second column) on trajectories generated by the expert. The x-axis for all plots is the number of training examples used to train the behavior cloning policy.}
    \label{fig:diagnosis_full}
\end{figure*}
In Fig~\ref{fig:diagnosis_full} we show the causal misidentification in several environments. We observe that while training and validation losses for behavior cloning are frequently near-zero for both the original and confounded policy, the confounded policy consistently yields significantly lower reward when deployed in the environment. This confirms the causal misidentification problem.

\section{DAgger with many more interventions}\label{app:dagger}
\begin{figure*}[htb]
    \centering
    \begin{subfigure}[b]{.48\textwidth}
        \includegraphics[width=\textwidth]{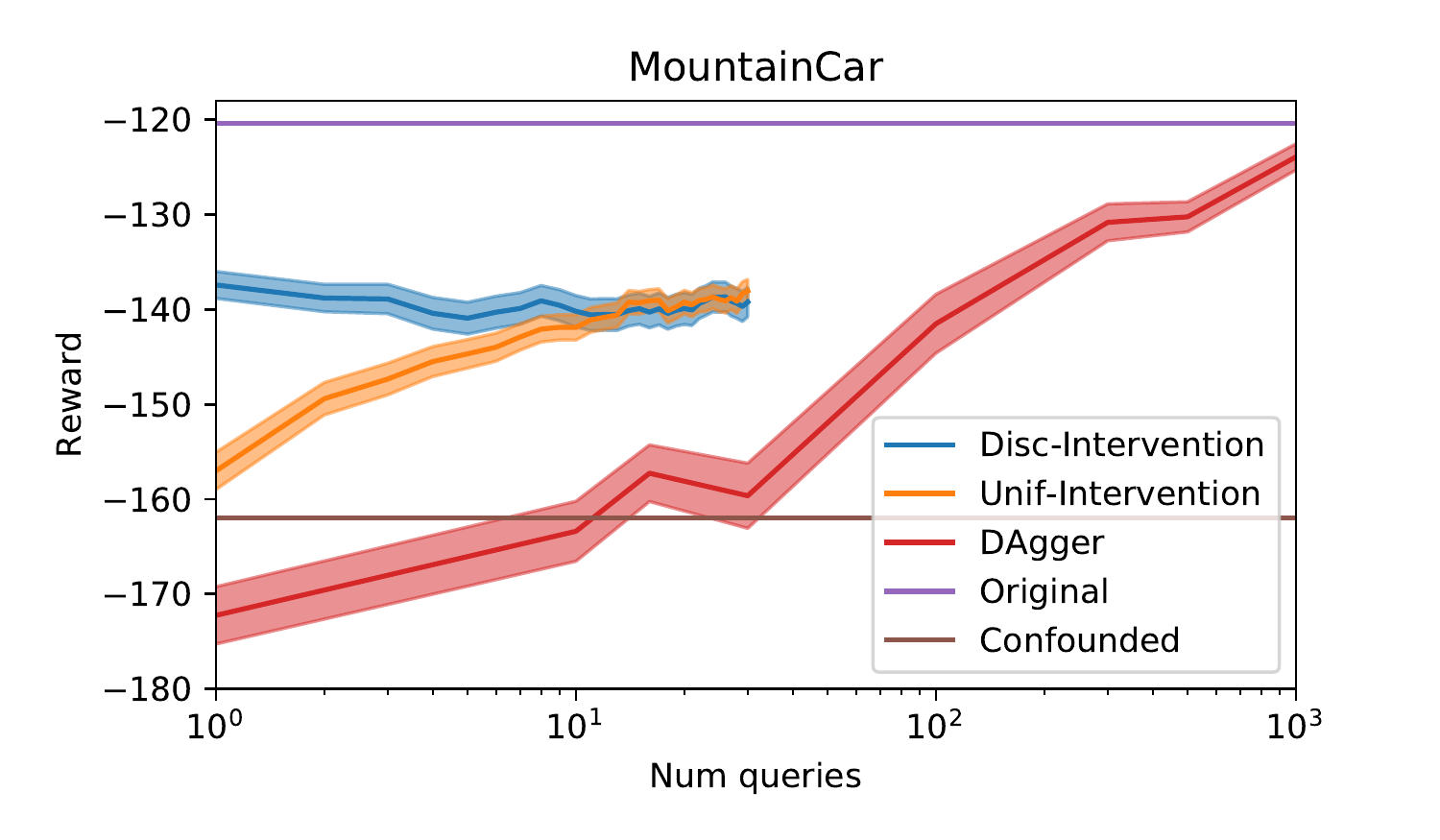}
        \caption{MountainCar}
    \end{subfigure}
    \begin{subfigure}[b]{.48\textwidth}
        \includegraphics[width=\textwidth]{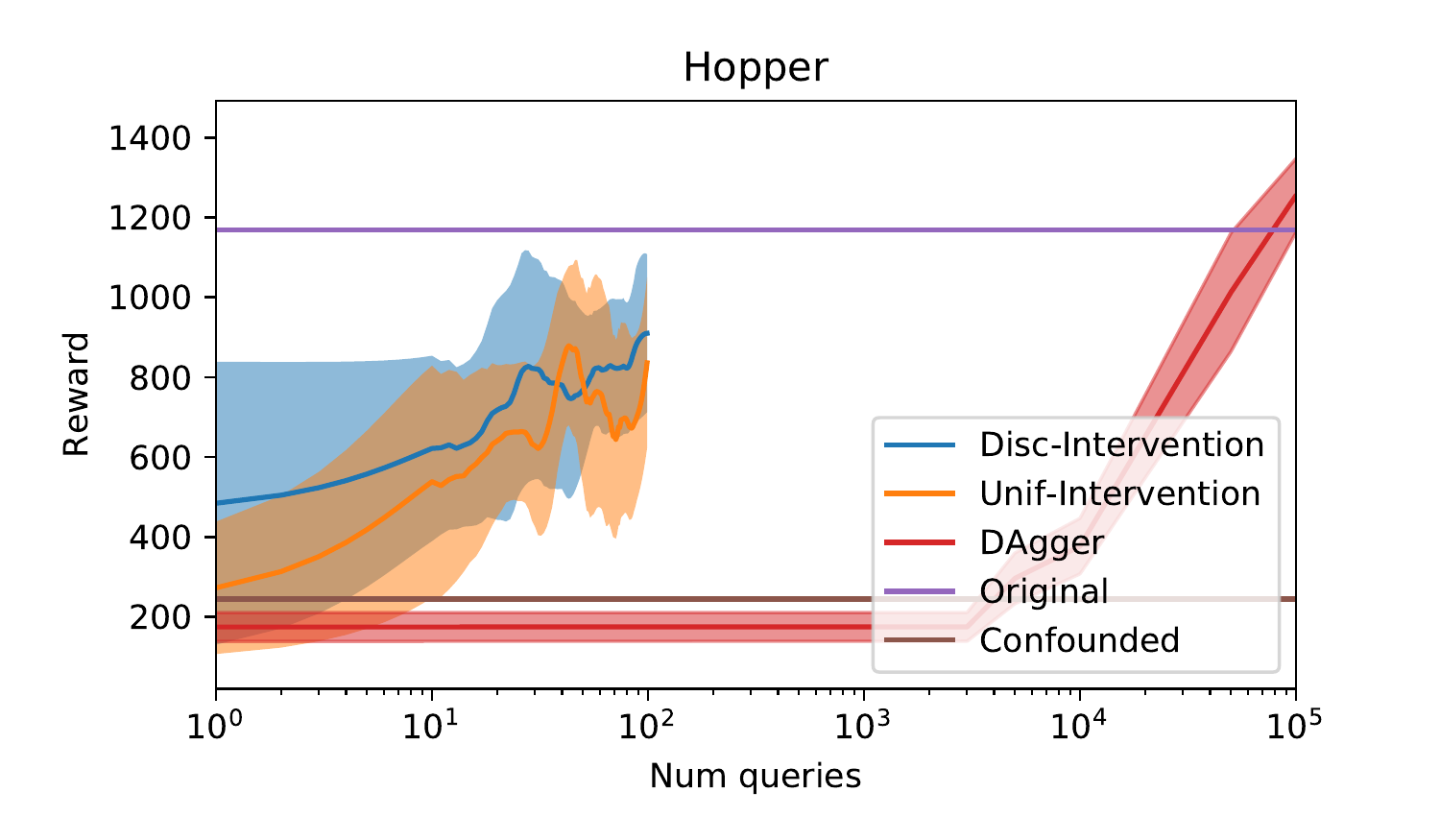}
        \caption{Hopper}
    \end{subfigure}
    \caption{DAgger results trained on the confounded state.}
    \label{fig:dagger}
\end{figure*}
In the main paper, we showed that DAgger performed poorly with equl number of expert interventions as our method. How many more samples does it need to do well? 

The results in Fig~\ref{fig:dagger} show that DAgger requires hundreds of samples before reaching rewards comparable to the rewards achieved by a non-DAgger imitator trained on the original state.

\section{GAIL Training Curves}\label{app:gail}
In Figure~\ref{fig:gail} we show the average training curves of GAIL on the original and confounded state. Error bars are 2 standard errors of the mean. The confounded and original training curve do not differ significantly, indicating that causal confusion is not an issue with GAIL. However, training requires many interactions with the environment.
\begin{figure}[ht!]
    \centering
    \includegraphics[width=0.5\textwidth]{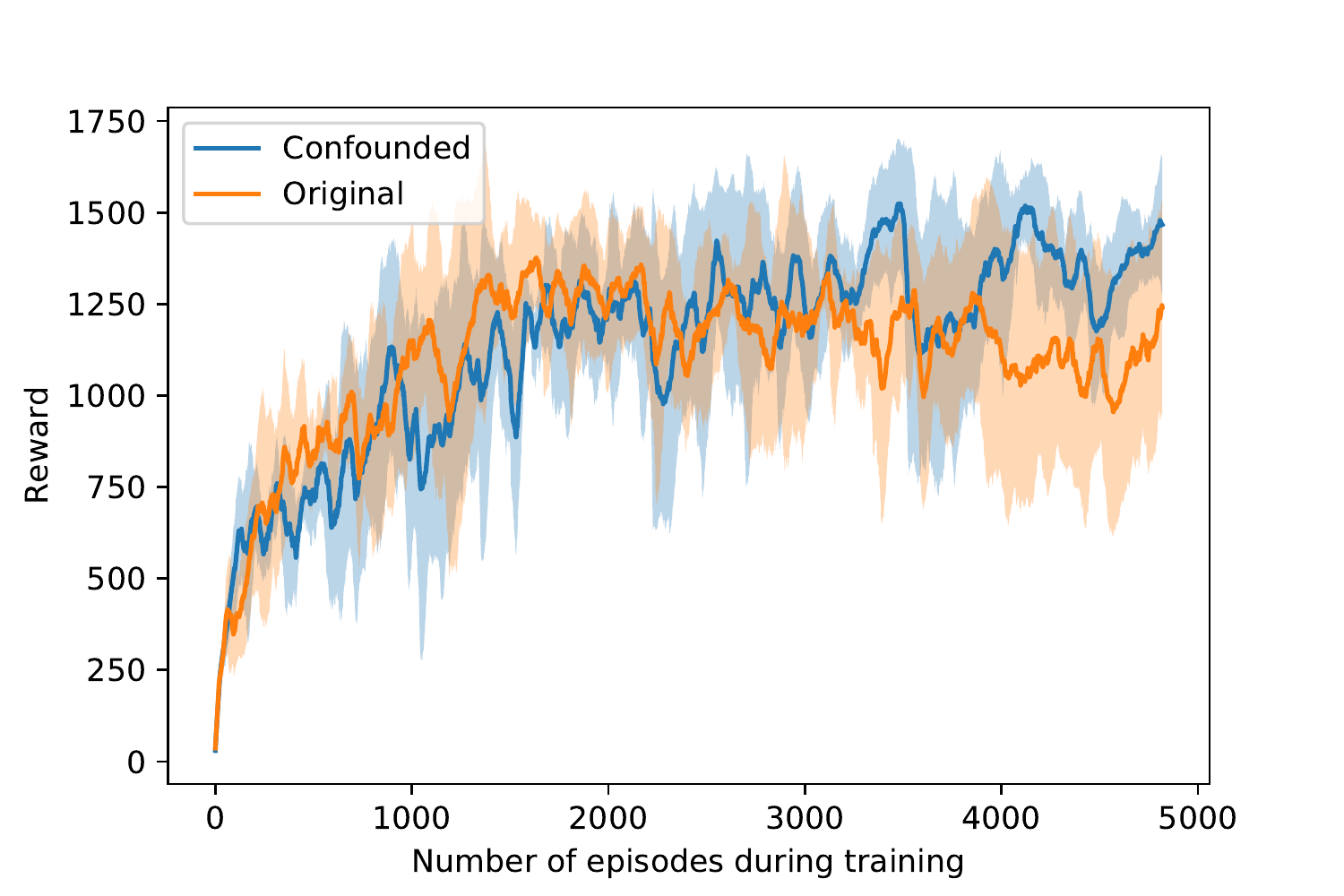}
    \caption{Rewards during GAIL training.}
    \label{fig:gail}
\end{figure}

\newpage
\section{Intervention Posterior Inference as Reinforcement Learning}\label{app:soft-q}

Given a method of evaluating the likelihood $p(\mathcal{O}|G)$ of a certain graph $G$ to be optimal and a prior $p_0(G)$, we wish to infer the posterior $p(G|\mathcal{O})$. The number of graphs is finite, so we can compute this posterior exactly. However, there may be very many graphs, so that impractically many likelihood evaluations are necessary. Only noisy samples from the likelihood can be obtained, as in the case of intervention through policy execution, where the reward is noisy, this problem is exacerbated.

If on the other hand, a certain structure on the policy is assumed, the sample efficiently can be drastically improved, even though policy can no longer be exactly inferred. This can be done in the framework of Variational Inference. For a certain variational family, we wish to find, for some temperature $\tau$:

\begin{align}
\pi(G) &= \argmin_{\pi(G)} D_{KL}(\pi(G)||p(\mathcal{O}|G)) \\
&=\argmin_{\pi(G)} \mathbb{E}_{\pi} \left[\log p(\mathcal{O}|G) + \log p_0(G) \right] + \tau \mathcal{H}_\pi(G) \label{eq:intervention-rl}
\end{align}

The variational family we assume is the family of independent distributions:
\begin{equation}
\pi(G)=\prod_i \pi_i(G_i)=\prod_i\text{Bernoulli}(G_i|\sigma(w_i/\tau))
\label{eq:independent-bernoulli}
\end{equation}

Eq~\ref{eq:intervention-rl} can be interpreted as a 1 step entropy-regularized MDP with reward $\tilde{r} = \log p(\mathcal{O}|G) + \log p_0(G)$ \cite{levine2018probrl}. It can be optimized through a policy gradient, but this would require many likelihood evaluations. More efficient is to use a value based method.  The independence assumption translates in a linear Q function: $Q(G) = \langle w, G \rangle + b$, which can be simply learned by linear regression on off-policy pairs $(G, \tilde{r})$. In Soft Q-Learning~\cite{haarnoja2017reinforcement} it is shown that the policy that maximizes Eq~\ref{eq:intervention-rl} is $\pi(G) \propto \exp Q(G)/\tau$, which can be shown to coincide in our case with Eq~\ref{eq:independent-bernoulli}:

\begin{align*}
    \pi(G) &= \frac{\exp(\langle w, G \rangle + b)/\tau}{\sum_{G'} \exp (\langle w, G' \rangle + b)/\tau}
    \propto \prod_i \exp(w_i G_i/\tau) \\
    \implies \pi(G) &= \prod_i\frac{\exp(w_i G_i/\tau)}{1+\exp w_i/\tau}=\prod_i\text{Bernoulli}(G_i|\sigma(w_i/\tau))
\end{align*}

\end{document}